\documentclass[10pt,journal,compsoc]{IEEEtran}

\usepackage[ruled,linesnumbered]{algorithm2e} 
\usepackage{url}
\usepackage{ifpdf}
\usepackage{graphicx} 
\usepackage{makecell}                                
\usepackage{amsmath,amsfonts,amssymb}                
\usepackage{cases}                                   
	\usepackage{stmaryrd}                                
	\usepackage{psfrag}
	\usepackage{placeins}                                
	\usepackage{colortbl}                                
	\usepackage{amsthm}                                  
	\usepackage{multirow,booktabs}
	\usepackage[table,xcdraw]{xcolor}
	\usepackage[normalem]{ulem}
	\useunder{\uline}{\ul}{}
	\usepackage[subfigure]{graphfig}
	\usepackage{footnote}
	\usepackage{epstopdf}
	\usepackage{array}
	\usepackage{ulem}
	\usepackage{float}
	\usepackage{ragged2e}
	\usepackage{hyperref}
	\ifCLASSOPTIONcompsoc
	\usepackage[nocompress]{cite}
	\else
	\usepackage{cite}
	\fi

	\ifCLASSINFOpdf
	\else
	\fi

\begin{document}
		
		\title{Balanced Multi-view Clustering}

		\author{Zhenglai Li,
			Jun Wang,
			Chang Tang,
			Xinzhong Zhu,
			Wei Zhang,
			Xinwang Liu
			
			\IEEEcompsocitemizethanks{
				
				\IEEEcompsocthanksitem The work was supported in part by the National Natural Science Foundation of China under grant 62476258 and 62325604, and in part by Taishan Scholar Program of Shandong Province in China under
Grant TSQN202312230.
				\IEEEcompsocthanksitem Z. Li is with the faculty of data science, City University of Macau, Macau 999078, China.
				\protect\\
				E-mail: zlli@cityu.edu.mo.
				\IEEEcompsocthanksitem J. Wang and X. Liu are with the school of computer, National University of Defense Technology, Changsha 410073, China.
				\protect\\
				E-mail: \{wang\_jun, xinwangliu\}@nudt.edu.cn.
				\IEEEcompsocthanksitem C. Tang is with the school of computer, China University of Geosciences, Wuhan 430074, China.
				\protect\\
				E-mail: tangchang@cug.edu.cn.
				\IEEEcompsocthanksitem X. Zhu is with College of Mathematics, Physics and Information Engineering, Zhejiang Normal University, Jinhua, China, and also with the Research Institute of Ningbo Cixing Co. Ltd, Ningbo, China. 
				\protect\\
				E-mail: zxz@zjnu.edu.cn.
				\IEEEcompsocthanksitem W. Zhang is with the Key Laboratory of Computing Power Network and Information Security, Ministry of Education, Shandong Computer Science Center (National Supercomputer Center in Jinan), Qilu University of Technology (Shandong Academy of Sciences), and is also with Shandong Provincial Key Laboratory of Computing Power Internet and Service Computing, Shandong Fundamental Research Center for Computer Science, Jinan 250000, China.
				\protect\\
				E-mail: wzhang@qlu.edu.cn.}
			\thanks{Manuscript received April 19, 2005; revised August 26, 2015. (Corresponding author: Chang Tang)}
		}
		
		\markboth{Journal of \LaTeX\ Class Files,~Vol.~14, No.~8, August~2015}%
		{Shell \MakeLowercase{\textit{et al.}}: Bare Advanced Demo of IEEEtran.cls for IEEE Computer Society Journals}
		
		\IEEEtitleabstractindextext{%
			\begin{abstract}
				\justifying
				Multi-view clustering (MvC) aims to integrate information from different views to enhance the capability of the model in capturing the underlying data structures.  The widely used joint training paradigm in MvC is potentially not fully leverage the multi-view information, since the imbalanced and under-optimized view-specific features caused by the uniform learning objective for all views.  For instance, particular views with more discriminative information could dominate the learning process in the joint training paradigm, leading to other views being under-optimized.  To alleviate this issue, we first analyze the imbalanced phenomenon in the joint-training paradigm of multi-view clustering from the perspective of gradient descent for each view-specific feature extractor.  Then, we propose a novel balanced multi-view clustering (BMvC) method, which introduces a view-specific contrastive regularization (VCR) to modulate the optimization of each view.  Concretely, VCR preserves the sample similarities captured from the joint features and view-specific ones into the clustering distributions corresponding to view-specific features to enhance the learning process of view-specific feature extractors.  Additionally, a theoretical analysis is provided to illustrate that VCR adaptively modulates the magnitudes of gradients for updating the parameters of view-specific feature extractors to achieve a balanced multi-view learning procedure.  In such a manner, BMvC achieves a better trade-off between the exploitation of view-specific patterns and the exploration of view-invariance patterns to fully learn the multi-view information for the clustering task.  Finally, a set of experiments are conducted to verify the superiority of the proposed method compared with state-of-the-art approaches on eight benchmark MvC datasets. The demo code of this work is publicly available at \url{https://github.com/guanyuezhen/BMvC}.
			\end{abstract}
			
			\begin{IEEEkeywords}
				Multi-view Clustering, Multi-view Learning, View Cooperation, Contrastive Learning, Balanced Learning.
		\end{IEEEkeywords}}

		\maketitle

		\IEEEdisplaynontitleabstractindextext

		%
		\IEEEpeerreviewmaketitle

		\ifCLASSOPTIONcompsoc
		\IEEEraisesectionheading{\section{Introduction}\label{sec:introduction}}
		\else
		\section{Introduction}
		\label{sec:introduction}
		\fi

		%
		%
		%
		%
		\IEEEPARstart{I}n real-world applications, data are often collected from multiple sensors or depicted by diverse feature descriptors. For example, the information captured from different sensors (e.g., cameras, LiDARs, Radars) is fused to achieve comprehensive scene understanding for autonomous vehicles~\cite{caesar2020nuscenes}. The text, visual, and audio are jointly utilized in the video process~\cite{li2024mvbench}. The traditional shallow features descriptors, e.g., histogram of oriented gradients (HOG)~\cite{dalal2005histograms}, and deep features descriptors, e.g., transformer~\cite{han2022survey}, are both leveraged to depict the objects in some computer vision tasks~\cite{wei2022masked}. Accordingly, multi-view learning methods, which explore the complementary information from diverse views to improve the model performance, have witnessed great progress in these years~\cite{zhou2024survey, zou2023dpnet, zou2024look}. 

		\begin{figure}[t]
			\centering
			\includegraphics[width = \linewidth]{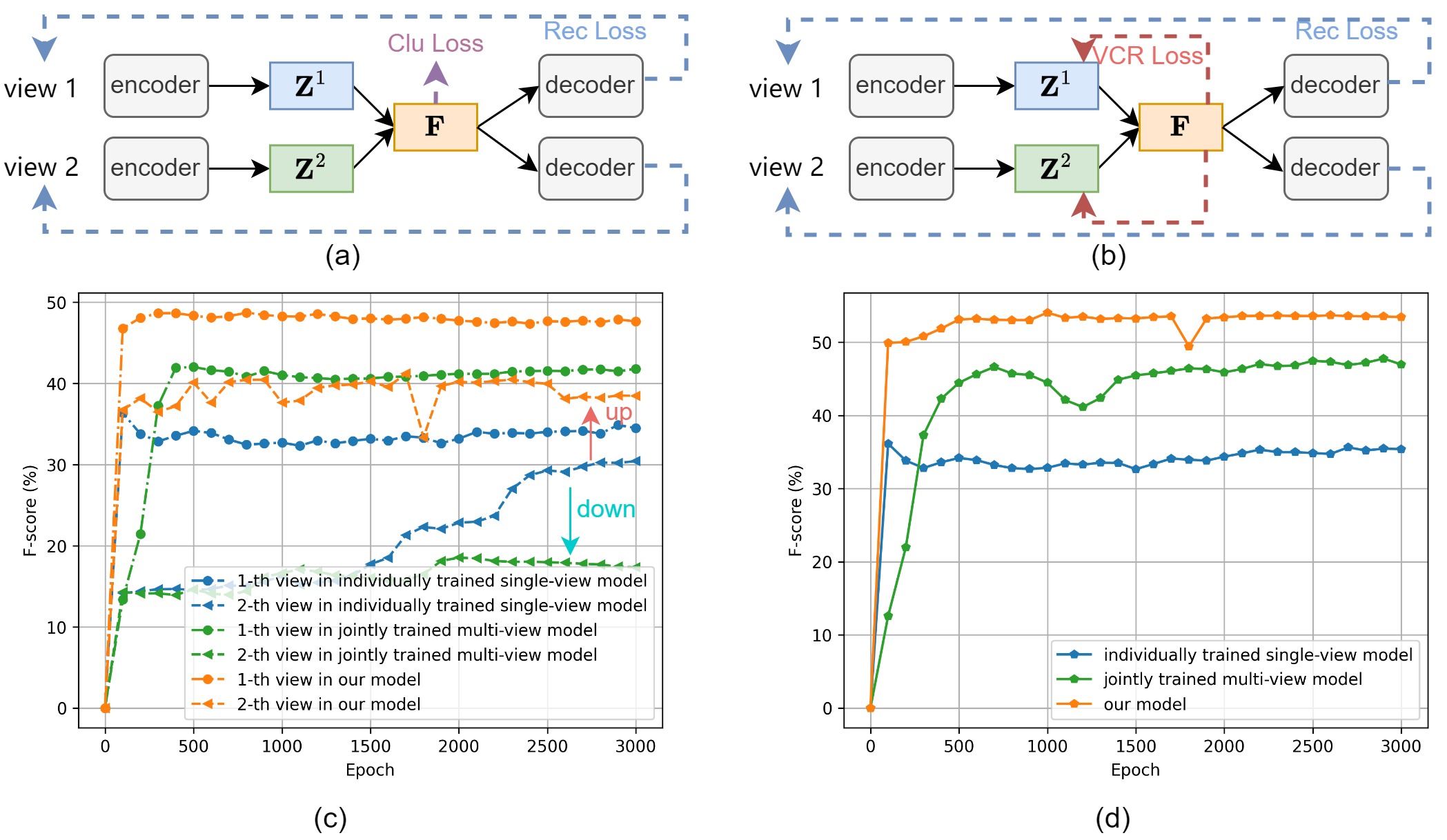}
			\caption{Illustration of imbalanced multi-view clustering issue in a two-view case. (a) The joint training multi-view clustering paradigm. (b) Our proposed balanced multi-view clustering approach. (c) The clustering performance of view-specific features extracted by diverse models. (d) The clustering performance of joint representation obtained by different models.}
			\label{fig:intro}
		\end{figure}
		
		Clustering algorithms aim to group the data points into their respective clusters without labels. As a useful unsupervised learning tool to reveal the underlying semantic structure of data for capturing meaningful information, clustering algorithms have been widely studied and employed in various fields, such as social network~\cite{block2020social}, image segmentation~\cite{liang2024clusterfomer}, and single-cell RNA-seq data clustering~\cite{kiselev2019challenges}. As one of the kinds of multi-view learning approaches, multi-view clustering expects to learn a shared and comprehensive clustering presentation by utilizing the complementary information from diverse views, to boost the clustering performance, which has received much attention recently~\cite{fang2023comprehensive}.
		
		Existing multi-view clustering approaches usually employ a joint training paradigm to optimize the clustering models as shown in Fig.~\ref{fig:intro} (a), which integrates the view-specific features into a shared representation under a uniform joint training objective, i.e., reconstructing the features (Rec Loss) or clustering (Clu Loss) by the shared representation, to extract the comprehensive information for the clustering task. For example, Xie et al.~\cite{xie2020joint} designed a joint learning framework, which learns diverse feature representations for images via different neural networks and fuses the information to obtain the clustering results. The work in~\cite{trosten2023effects}, designs a framework to investigate the effects of self-supervision and contrastive alignment in deep multi-view clustering. In recent studies, some researchers claimed that the clustering performance of the multi-view model may degrade when the number of views increases~\cite{xu2024investigating, wang2024multiple}. The main reason is that the particular views with noise information could be not only useless but also even detrimental for the clustering task~\cite{xu2024investigating}. This discrepancy makes it challenging for the multi-view clustering model to effectively learn from all views simultaneously under a uniform joint training objective. To alleviate this issue, previous methods usually select some informative views for clustering~\cite{wang2024multiple, wangevaluate}, but it also leads to insufficient utilization of the complementary information among all views. 
		
		Beyond the clustering performance degradation caused by diverse view qualities, we observe that even when the multi-view clustering models perform better than their single-view counterparts, they still fail to fully explore the potential of multiple views. As shown in Fig.~\ref{fig:intro}, a set of experiments are conducted on the nuswide dataset~\cite{zhen2019deep} to evaluate the performance of diverse multi-view clustering settings. From the results, we note that the joint training multi-view clustering models outperform than single-view ones, but the clustering performance of view-specific features within the jointly-trained multi-view clustering models performs similarly even worse than the performance of view-specific features captured from solely trained single-view clustering models. For example, in Fig.~\ref{fig:intro} (c), $1$-th view features in the jointly-trained multi-view clustering model exists a clear clustering performance drop, compared with that in the single-view model. Such observations indicate that the view-specific feature extractors are under-optimized with an imbalanced degree in the joint training multi-view clustering model. The main reason is that the multi-view datasets often contain some views with more discriminative information, which tend to be favored and dominant in the training procedure, consequently suppressing the learning process of other views. Such view preference within the datasets causes the observed imbalanced multi-view clustering issue among different views.
		
		Although many solid MVC methods have been proposed~\cite{fang2023comprehensive, zhou2024survey}, the efforts to tackle the imbalance multi-view clustering issue are still limited. To this end, we propose BMvC (short for balanced multi-view clustering) as illustrated in Fig.~\ref{fig:intro} (b), a novel solution that introduces a view-specific contrastive regularization (VCR) to modulate the optimization of each view. Concretely, VCR constrains the clustering distributions of view-specific features to preserve the sample similarities captured from the joint features and view-specific ones. As a result, the VCR provides an additional gradient for updating view-specific feature extractors so that a better trade-off between the exploitation of view-specific patterns and exploration of view-invariance patterns can be achieved to fully learn the multi-view information for the clustering task. As shown in Fig.~\ref{fig:intro} (c), the clustering performance of view-specific features extracted by our proposed method outperforms better than that of joint training and single-view training models. Finally, the proposed method achieves better results compared with the joint training model as given in Fig.~\ref{fig:intro} (d).
		The contributions are summarized as follows,
		\begin{itemize}
			\item We observe and analyze the imbalanced learning phenomenon in multi-view clustering from the gradient descent perspective that view-specific encoders in joint training paradigm are imbalanced under-optimized and certain views could be worse optimized than others in the training procedure.
			\item We propose BMvC to achieve a better trade-off between the exploitation of view-specific patterns and the exploration of view-invariance patterns to fully learn the multi-view information for the clustering task.
			\item We introduce VCR to exploit the sample similarities captured from the joint features and view-specific ones to regularize the clustering distributions of view-specific features being structural so that the view-specific encoders are learned balanced. A theoretical analysis is formulated to demonstrate that VCR adaptively modulates the gradients for updating the parameters of view-specific encoders to achieve a balanced learning processing.
			\item We perform extensive experiments on eight benchmark multi-view datasets to verify the effectiveness of the proposed method.
		\end{itemize}
		
		The rest of this paper is organized as follows. Section 2 gives a brief review of the most related work.  In section 3, we present the details of the proposed BMvC method. Section 4 provides a series of experimental results and discussions. In Section 5, we provide a conclusion of this paper.
		
		\section{Related Work}
		
		\subsection{Multi-view Clustering}
		Multi-view clustering, which aggregates comprehensive information from various views, has attracted considerable attention due to the increasing amount of multi-view data. Current methods can be classified into three primary categories based on their approach to generating clustering results: matrix factorization-based methods~\cite{wang2018multiview, yang2020uniform, zhao2017multi}, graph-based methods~\cite{zhang2018generalized, kang2020large, zhang2020consensus, wang2019gmc, lin2021graph, li2015large}, and deep learning-based methods~\cite{xu2021deep, wang2018partial, xie2020joint, wangevaluate}.
		
		Matrix factorization-based methods aim to derive a shared feature representation from multiple views through matrix factorization techniques. For instance, Liu et al.~\cite{liu2013multi} developed a shared factorization that provides compatible solutions across diverse views to achieve clustering outcomes. Building upon the deep matrix factorization methods~\cite{trigeorgis2016deep}, Zhao et al.~\cite{zhao2017multi} hierarchically decomposed the multi-view representations into a shared latent feature space, progressively learning complementary information from the various views. They employed the Laplacian regularization to maintain the locality properties of the data within this latent space. Yang et al.~\cite{yang2020uniform} introduced a tri-factorization-based non-negative matrix factorization approach to decompose multi-view data into a uniform distribution, thereby enhancing the separability of the learned consensus representation. To address the challenges posed by the incomplete multi-view data, Wen et al.~\cite{wen2023graph} implemented an adaptive feature weighting constraint aimed at mitigating the effects of redundant and noisy features during the matrix factorization process. Additionally, they designed a graph-embedded consensus representation learning term to preserve the structural information inherent in incomplete multi-view data.
		
		Graphs are widely used data structures for representing the relationships among different samples. Graph-based multi-view clustering methods integrate graph similarities learned from various views, ultimately deriving clustering results through spectral clustering techniques. Self-representation~\cite{elhamifar2013sparse} and adaptive neighbor graph learning~\cite{nie2014clustering} are two prevalent approaches for constructing high-quality similarity graphs in prior works. For instance, Zhang et al.~\cite{zhang2018generalized} jointly learned both latent and subspace representations, employing neural networks to enhance the latent representation learning process and improve model generalization. Nie et al.~\cite{nie2017multi} extended adaptive neighbor graph learning to a multi-view setting and captured a Laplacian rank-constrained consensus graph from multiple views for the clustering purpose. Tang et al.~\cite{tang2022unified} unified $k$-means and spectral clustering to leverage information from both graphs and embedding matrices, thereby enhancing clustering performance. In~\cite{qiang2021fast}, a joint learning framework that simultaneously learns the graph embedding matrix and clustering indicators, is developed to effectively unify the spectral embedding and spectral rotation manners. Additionally, they further employed an anchor graph to reduce the computational complexity of the model.
		
		Deep learning-based methods leverage the powerful representation capabilities of deep neural networks to derive consensus clustering results from multi-view data. Wang et al.~\cite{wang2018partial} utilized a consistent generative adversarial network to capture a shared representation from incomplete multi-view data for clustering. Xie et al.~\cite{xie2020joint} developed a joint learning framework that employs different neural networks to extract features for comprehensively depicting each image. This framework enables simultaneous feature embedding, multi-view information fusion, and data clustering to achieve joint training. Xu et al.~\cite{xu2021deep} introduced a collaborative training scheme to capture complementary information from diverse views, facilitating the joint learning of feature representations and cluster assignments. In recent years, self-supervised learning with contrastive loss has made significant strides and received rapid developments across various fields~\cite{krishnan2022self, liu2021self}. Drawing inspiration from the robust feature learning capabilities of self-supervised learning, contrastive loss has been extensively applied in multi-view clustering. For example, Xu et al~\cite{xu2022multi} aligned multi-view information from both high-level semantics and low-level features through contrastive learning, effectively capturing the common semantics for clustering. Lin et al.~\cite{lin2021completer} integrated feature representation learning and missing sample recovery into a unified framework, using contrastive learning to capture informative and consistent representations from different views. Additionally, the work in~\cite{trosten2023effects} designed a framework to explore the effectiveness of self-supervision and contrastive alignment in the multi-view clustering task.
		
		While previous multi-view clustering methods have made significant strides in improving clustering performance from diverse perspectives, the issue of imbalanced learning in multi-view clustering has been rarely studied. To this end, we propose a balanced multi-view clustering method aimed at enhancing cooperation among views.
		
		\subsection{Imbalanced Learning}
		Multi-view and multi-modal learning are two closely intertwined concepts. In the context of multi-modal fusion, studies have shown that joint-training supervised multi-modal learning often encounters modality competition~\cite{huang2022modality, huareconboost}. This issue arises when different modalities are optimized synchronously under a unified objective. During training, modalities with more discriminative information tend to dominate the learning process, converging more quickly than others. As a result, the other modalities struggle to update their learning parameters fully. Consequently, the under-optimized uni-modal feature extractors prevent the full exploitation of modality-specific information, generating a bottleneck in boosting the performance of multimodal learning.  
		
		Recently, several methods have emerged to address this issue~\cite{du2023uni, peng2022balanced, wei2024fly, fan2023pmr}. For instance, Du et al.~\cite{du2023uni} employed a knowledge distillation scheme for uni-modal distillation to enhance multimodal model performance. The works in~\cite{peng2022balanced, wei2024fly} utilize the adaptive gradient modulation to suppress the learning process of the dominant modality and facilitate the training of other modalities to achieve more balanced multi-modal learning. Fan et al.~\cite{fan2023pmr} leveraged a prototypical cross-entropy loss to accelerate the learning of weaker modalities. Additionally, the study in~\cite{kontras2024improving} introduced a multi-loss objective to dynamically adjust the learning process for each modality, refining the balancing mechanism. However, these previous methods primarily address the imbalanced learning problem in supervised tasks, making them infeasible to be directly applied for the unsupervised multi-view learning task.  
		
		\section{The Proposed Method}
		\textit{Problem Definition}: For convenience, the multi-view data is presented as $\{\mathbf{X}^r\in \mathbb{R}^{N\times D^r}  \}_{r=1}^M$, where $N$ and $M$ are the number of samples and views. $D^r$ is the feature dimensions of $r$-th view. The goal of multi-view clustering is to formulate a projection to capture the underlying data structures, i.e., $\{\mathbf{X}^r\in \mathbb{R}^{N\times D^r}  \}_{r=1}^M \to \mathbf{Y}$, where $\mathbf{Y}$ denotes the clustering results.
		
		In this section, we first analyze the imbalanced learning phenomenon in multi-view clustering. Then, we introduce the proposed balanced multi-view clustering framework. Finally, the analysis related to the proposed method is given.
		
		\subsection{Imbalanced Multi-view Clustering}\label{sec:imvca}
		
		\begin{figure*}[t]
			\centering
			\includegraphics[width = 1\linewidth]{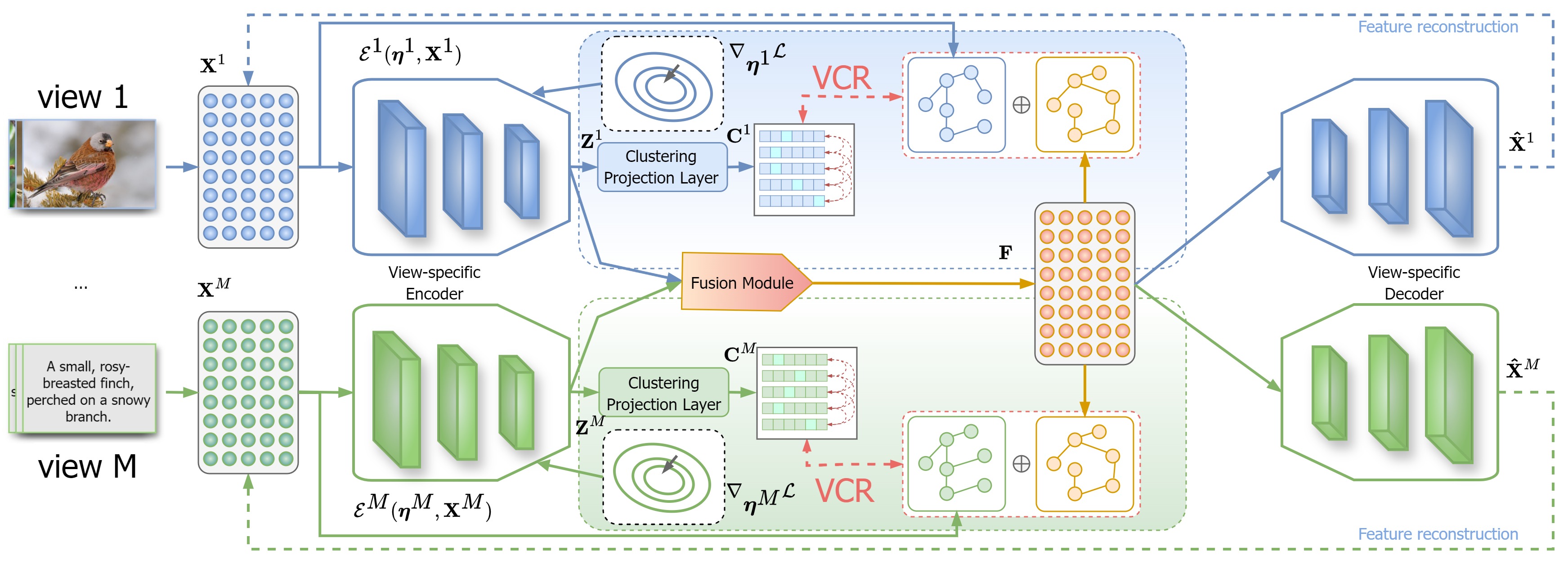}
			\caption{Illustration of balanced multi-view clustering (BMvC). The multi-view data are first processed through their encoders to extract view-specific latent features. Next, a fusion module aggregates the representations into a unified joint one, which is then used to reconstruct the multi-view data via view-specific decoders. Finally, the view-specific contrastive regularization (VCR) and feature reconstruction loss are employed to guide the optimization of the model.}
			\label{fig:BMvC}
		\end{figure*}

		In this part, we analyze the imbalanced phenomenon in the joint training paradigm of previous multi-view clustering methods. The view-specific encoder $\mathcal{E}^r(\boldsymbol{\eta}^r, \mathbf{X}^r)$ is utilized to extract $r$-th view-specific latent features $\mathbf{Z}^r$ from $\mathbf{X}^r$ as,
		\begin{equation}
			\mathbf{Z}^r = \mathcal{E}^r(\boldsymbol{\eta}^r, \mathbf{X}^r), r=1, 2, ..., M,
		\end{equation}
		where $\boldsymbol{\eta}^r$ is the parameters of $r$-th view-specific encoder. After the view-specific feature extraction, their features are usually fused by a feature concatenation operation, and passed through a linear layer to obtain the joint representation $\mathbf{F}$, as follows,
		\begin{equation}\label{eq:feature_fusion}
			\mathbf{F} = \mathcal{F}(\boldsymbol{\delta}, [\mathbf{Z}^1, ..., \mathbf{Z}^M]),
		\end{equation}
		where $\mathcal{F}(\boldsymbol{\delta}, [\mathbf{Z}^1, ..., \mathbf{Z}^M])$ with its parameters $\boldsymbol{\delta}$ represents the fusion module. $[\cdot]$ is a feature concatenation operation. By decomposing $\boldsymbol{\delta}$ as $\boldsymbol{\delta} = [\boldsymbol{\delta}^1; ...; \boldsymbol{\delta}^M]$, Eq.~\eqref{eq:feature_fusion} can be rewritten as,
		\begin{equation}\label{eq:feature_fusion2}
			\mathbf{F} = \sum_{s=1}^M \boldsymbol{\delta}^s \cdot \mathbf{Z}^s.
		\end{equation}
		
		Finally, the joint representation $\mathbf{F}$ passes to the view-specific decoder $\mathcal{D}^r(\boldsymbol{\theta}^r, \mathbf{Z}^r)$ to reconstruct the view-specific features. The objective of a multi-view clustering network is usually to minimize the reconstruction loss between the original and reconstructed features, and additionally constrains the joint representation $\mathbf{F}$ with a clustering loss,
		\begin{equation}\label{eq:jtp_loss}
			\begin{split}
				\mathcal{L} = & \sum_{r=1}^M \ell(\mathbf{\hat{X}}^r, \mathbf{X}^r) + \beta \psi(\mathbf{F})\\
				= &\sum_{r=1}^M \ell(\mathcal{D}^r(\sum_{s=1}^M \boldsymbol{\delta}^s \cdot \mathcal{E}^s(\boldsymbol{\eta}^s, \mathbf{X}^s)), \mathbf{X}^r) \\
				& + \beta \psi(\sum_{s=1}^M \boldsymbol{\delta}^s \cdot \mathcal{E}^s(\boldsymbol{\eta}^s, \mathbf{X}^s)),
			\end{split}
		\end{equation}
		where $\beta$ is a hyper-parameter to balance reconstruction loss $\ell(\cdot)$ and clustering loss $\psi(\cdot)$. During the optimization procedure, the parameters $\boldsymbol{\eta}^s$ in $\mathcal{E}^s(\boldsymbol{\eta}^s, \mathbf{X}^s)$ are updated as,
		\begin{equation}\label{eq:update_eta_jl}
			\begin{split}
				\boldsymbol{\eta}^{s, t+1} 
				= \ & \boldsymbol{\eta}^{s, t} - \gamma \nabla_{\boldsymbol{\eta}^{s, t}}\mathcal{L}\\
				= \ & \boldsymbol{\eta}^{s, t} - \gamma \Big( \Delta \Big),\\
				\text{where} \ \Delta = \ & \Big(\sum_{r=1}^M \frac{\partial \ell(\mathcal{D}^r(\mathbf{F}), \mathbf{X}^r) }{\partial \mathcal{D}^r(\mathbf{F})} \frac{\partial \mathcal{D}^r(\mathbf{F}) }{\partial \mathbf{F}} + \beta \frac{\partial \psi(\mathbf{F}) }{\partial \mathbf{F}} \Big) \\
				& \frac{\partial \boldsymbol{\delta}^s \cdot \mathcal{E}^s(\boldsymbol{\eta}^{s, t}, \mathbf{X}^s) }{\partial \boldsymbol{\eta}^{s, t}}.
			\end{split}
		\end{equation}
		
		In Eq.~\eqref{eq:update_eta_jl}, $\sum_{r=1}^M \frac{\partial \ell(\mathcal{D}^r(\mathbf{F}), \mathbf{X}^r) }{\partial \mathcal{D}^r(\mathbf{F})} \frac{\partial \mathcal{D}^r(\mathbf{F}) }{\partial \mathbf{F}} + \beta \frac{\partial \psi(\mathbf{F}) }{\partial \mathbf{F}}$ is a shared term in gradient descent optimization procedure for parameters in all view-specific encoders. Thus, we can obtain the following observations: 1) The parameters update for encoders corresponding to all views will be stuck simultaneously, when the gradient for the shared term varies small determined by the joint representation $\mathbf{F}$; 2) If certain view, such as the $r$-th view, consists of more discriminative information than other views, it would dominate the multi-view fusion process for the joint representation $\mathbf{F}$ and the gradient $\sum_{r=1}^M \frac{\partial \ell(\mathcal{D}^r(\mathbf{F}), \mathbf{X}^r) }{\partial \mathcal{D}^r(\mathbf{F})} \frac{\partial \mathcal{D}^r(\mathbf{F}) }{\partial \mathbf{F}} + \beta \frac{\partial \psi(\mathbf{F}) }{\partial \mathbf{F}}$. Even if other views are under-optimized and their underlying clustering structures are not fully aggregated in the joint representation $\mathbf{F}$, the information from the dominated view can still reconstruct the features. To address this issue, we intend to find a better trade-off between the exploitation of view-specific patterns and the exploration of view-invariance patterns to fully learn the multi-view information for the clustering task.

		\subsection{Balanced Multi-view Clustering}
		
		Our proposed balanced multi-view clustering (BMvC) method follows the same network architecture as the joint training paradigm approaches depicted in Section~\ref{sec:imvca}. As shown in Fig.~\ref{fig:BMvC}, the multi-view data $\{\mathbf{X}^r\in \mathbb{R}^{N\times D^r}  \}_{r=1}^M$ are processed by their view-specific encoders $\{\mathcal{E}^r(\boldsymbol{\eta}^r, \mathbf{X}^r)\}_{r=1}^M$ to extract the view-specific latent features $\{\mathbf{Z}^r\}_{r=1}^M$. Then, the fusion module $\mathbf{F} = \sum_{s=1}^M \boldsymbol{\delta}^s \cdot \mathbf{Z}^s$ is introduced to aggregate the multi-view representations into a joint one, which is utilized to reconstruct the multi-view data by view-specific decoders $\{\mathcal{D}^r(\boldsymbol{\theta}^r, \mathbf{F})\}_{r=1}^M)$. To boost the exploitation of view-specific patterns and the exploration of view-invariance patterns, we introduce a view-specific contrastive regularization (VCR), which aims to maximize the similarities of view-specific clustering indicators within their neighborhood ones based on the similarities extracted from the view-specific features and the aggregated joint ones. In such a paradigm, the feature extraction capabilities of the view-specific encoders are enhanced since the VCR introduces an extra gradient to optimize the parameters of view-specific encoders, which is analyzed in Section~\ref{sec:trga}. Thus, the total objective function of the proposed method is,
		\begin{equation}\label{eq:rec_loss}
			\begin{split}
				\mathcal{L} = &\  \sum_{r=1}^M \ell(\mathbf{\hat{X}}^r, \mathbf{X}^r) +  \sum_{r=1}^M \lambda \zeta(\mathbf{F}, \mathbf{Z}^r)
				\\
				=\ & \sum_{r=1}^M \ell(\mathcal{D}^r(\sum_{s=1}^M \boldsymbol{\delta}^s \cdot \mathcal{E}^s(\boldsymbol{\eta}^s, \mathbf{X}^s)), \mathbf{X}^m)\\
				+ \ & \lambda \sum_{r=1}^M \zeta(\sum_{s=1}^M \boldsymbol{\delta}^s \cdot \mathcal{E}^s(\boldsymbol{\eta}^s, \mathbf{X}^s), \mathcal{E}^r(\boldsymbol{\eta}^r, \mathbf{X}^r)),
			\end{split}
		\end{equation}
		where $\lambda$ is trade-off hyper-parameter to balance reconstruction loss $\ell(\cdot)$ and view-specific contrastive regularization loss $\zeta(\cdot)$. Finally, the clustering results is obtained by conducting K-means clustering on the learned joint representation $\mathbf{F}$.
		
		\subsubsection{View-specific Contrastive Regularization}
		Preserving the similarity from the original feature space is a useful scheme to enhance the representation capability of the latent features~\cite{he2003locality, cha2024honeybee, tang2021cross}. Inspired by this, we tend to constrain the view-specific feature representation with structural information to boost their representation capabilities. To this end, we formulate three key components, i.e., the clustering projection layer, similarity graph construction, and contrastive loss.
		
		\textbf{Clustering Projection Layer:}
		We project the view-specific representations $\{\mathbf{Z}^r\}_{r=1}^M$ into the clustering space by a linear layer. Then, we employ the QR decomposition~\cite{gander1980algorithms} to ensure the orthogonal property of clustering indicators in the clustering space. As a result, the view-specific clustering indicator matrices $\{\mathbf{C}^s\}_{s=1}^M$, ensuring the orthogonality property in the clustering space, are obtained to reveal the view-specific cluster distributions.
		
		\textbf{Similarity Graph Construction:}
		Similarity graphs are widely used to depict the relationships among different samples. In this part, we employ the similarities to enhance the feature discriminative capability of the view-specific features. The CAN~\cite{nie2014clustering} is a useful adaptive neighbor graph learning manner, which measures the similarities among different samples as the link probabilities,
		\begin{equation}\label{eq:can1}
			\begin{split}
				\min_{\mathbf{P}} \ & \sum_{i,j=1}^N(\Vert \mathbf{q}_i - \mathbf{q}_j \Vert_2^2)p_{ij} + \alpha p_{ij}^2)\\
				s.t. \ & \sum_{j=1}^n p_{ij} = 1,  p_{ij} > 0,
			\end{split}
		\end{equation}
		where $\alpha$ is a hyper-parameter. $p_{ij}$ denotes the similarity between sample $\mathbf{q}_i$ and $\mathbf{q}_j $. According to the Lagrangian Augmention and the Karush-Kuhn-Tucker condition, the closed-form solution of Eq.~\eqref{eq:can1} can be obtained as~\cite{nie2014clustering},
		\begin{equation}\label{eq:can2}
			p_{ij} = \frac{\hat{d}_{i,k+1} - \hat{d}_{ij}}{k\hat{d}_{i,k+1} - \sum_{j=1}^k \hat{d}_{ij}}
		\end{equation}
		where $k$ is the number of nearest neighbors, $\hat{d}_{ij}$ represents the $ij$-th element in matrix $\hat{\mathbf{D}}$, which is formulated by sorting matrix $\mathbf{D}$, with its $ij$-th element $d_{ij} = \Vert \mathbf{q}_i - \mathbf{q}_j \Vert_2^2$, from small to large. 
		
		Based on the similarity graph construction manner in Eq.~\eqref{eq:can2}, we can formulate the view-specific graphs $\{\mathbf{G}^s\}_{s=1}^M$ and joint graph $\bar{\mathbf{G}}$ from the view-specific original features $\{\mathbf{X}^s\}_{s=1}^M$ and joint features $\mathbf{F}$. The joint graph $\bar{\mathbf{G}}$ depicts the sample relationships with the multi-view information, while the view-specific graphs $\{\mathbf{G}^s\}_{s=1}^M$ measure the correlations among samples only with their view-specific information. To simultaneously keep the view-specific and view-joint knowledge, the view-specific graphs $\{\mathbf{G}^s\}_{s=1}^M$ and joint graph $\bar{\mathbf{G}}$ are fused as follows,
		\begin{equation}
			\mathbf{\hat{G}}^s = \frac{\mathbf{G}^s + \bar{\mathbf{G}}}{2}, s=1, ..., M,
		\end{equation}
		where $\mathbf{\hat{G}}^s$ denotes the enhanced $s$-th view-specific graph.
		
		\textbf{Contrastive Loss:}
		The contrastive learning approaches push the samples far away from the negative anchors, while pull in the samples with positive anchors~\cite{chen2020simple, wang2021understanding}. The view-specific similarity graphs $\{\mathbf{\hat{G}}^s\}_{s=1}^M$ provide the positive and negative pairs for contrastive learning corresponding to the view-specific clustering indicators $\{\mathbf{C}^s\}_{s=1}^M$. To keep the structural information into the view-specific clustering indicators, we first compute the sample correlations via the cosine similarity as,
		\begin{equation}
			a_{ij}^s = \frac{\mathbf{c}_i^s\mathbf{c}_j^{s\top}}{\Vert \mathbf{c}_i^s \Vert \Vert \mathbf{c}_j^s \Vert},
		\end{equation}
		where $a_{ij}^s \in [0, 1]$ is the correlation between $\mathbf{c}_i^s$ and $\mathbf{c}_j^s$. Ideally, if the $\mathbf{c}_i^s$ and $\mathbf{c}_j^s$ are the positive pair, $a_{ij}^s$ is closed to $1$. And if $\mathbf{c}_i^s$ and $\mathbf{c}_h^s$ are the negative pair, $a_{ih}^s$ is closed to $0$. Thus, we can formulate the following objective function as,
		\begin{equation}\label{eq:vcrn_loss}
			\min_{\mathbf{C}^s} \  = \sum_{i=1}^N ( \sum_{j\in N^+} (1 - a_{ij}^s)^2 + \sum_{h\in N^-} (a_{ih}^s)^2 ),
		\end{equation}
		where $N^+$ and $N^-$ denote the positive and negative sets, respectively. Considering the view-specific similarity graphs $\{\mathbf{\hat{G}}^s\}_{s=1}^M$ measure the data relationships in the range $[0, 1]$ rather than the binary value, i.e., $0$ or $1$, we reformulate the Eq.~\eqref{eq:vcrn_loss} with a weighted manner as,
		\begin{equation}\label{eq:vcr_loss}
			\zeta = \frac{\sum_{i,j=1}^N ( \hat{g}^s_{ij} (1 - a^s_{ij}) )^2 }{\sum_{i,j=1}^N \hat{g}^s_{ij}} + \frac{\sum_{i,j=1}^N ((1 - \hat{g}^s_{ij}) (a^s_{ij}))^2 }{\sum_{i,j=1}^N (1 - \hat{g}^s_{ij})} .
		\end{equation}
		
		In Eq.~\eqref{eq:vcr_loss}, the first term is used to constrain the samples more closed to each other when they are in the same neighborhoods, while the second term encourages pushing the samples away from others as they are in distinct neighborhoods. In such a manner, the clustering distributions of view-specific features are well learned under the guidance of view-specific graphs and joint graph. Thus, the view-specific encoders can be further optimized to enhance their feature extraction capabilities.
		
		\subsubsection{The Resultant Gradient Analysis}\label{sec:trga}
		During the optimization procedure, gradient from the view-specific contrastive regularization for update the parameters $\boldsymbol{\eta}^s$ in $\mathcal{E}^s(\boldsymbol{\eta}^s,\mathbf{X^s})$ is represented as,
		\begin{equation}\label{eq:update_eta_vcr}
			\begin{split}
				& \nabla_{\boldsymbol{\eta}^{s}} \sum_{r=1}^M \zeta(\sum_{s=1}^M \boldsymbol{\delta}^s \cdot \mathcal{E}^s(\boldsymbol{\eta}^s, \mathbf{X}^s), \mathcal{E}^r(\boldsymbol{\eta}^r, \mathbf{X}^r))
				\\
				= \ & \frac{\partial \zeta(\mathbf{F}, \mathbf{Z}^s) }{\partial \mathbf{F}} \frac{\partial \boldsymbol{\delta}^s \cdot \mathcal{E}^s(\boldsymbol{\eta}^{s, t}, \mathbf{X}^s) }{\partial \boldsymbol{\eta}^{s, t}}\\
				+ \ & \frac{\partial \zeta(\mathbf{F}, \mathbf{Z}^s) }{\partial \mathbf{Z}^s} \frac{\partial \mathcal{E}^s(\boldsymbol{\eta}^{s, t}, \mathbf{X}^s) }{\partial \boldsymbol{\eta}^{s, t}}.
			\end{split}
		\end{equation}
		
		Thus, the parameters $\boldsymbol{\eta}^s$ in $\mathcal{E}^s(\boldsymbol{\eta}^s, \mathbf{X}^s)$ are updated as,
		\begin{equation}\label{eq:update_eta_total}
			\begin{split}
				\boldsymbol{\eta}^{s, t+1} = \ & \boldsymbol{\eta}^{s, t} - \gamma (\Delta_1 + \lambda \Delta_2),\\
				\text{where} \ \Delta_1 = \ & \Big(\sum_{r=1}^M \frac{\partial \ell(\mathcal{D}^r(\mathbf{F}), \mathbf{X}^r) }{\partial \mathcal{D}^r(\mathbf{F})} \frac{\partial \mathcal{D}^r(\mathbf{F}) }{\partial \mathbf{F}} \\
				+ \ & \beta \frac{\partial \psi(\mathbf{F}) }{\partial \mathbf{F}} \Big) \frac{\partial \boldsymbol{\delta}^s \cdot \mathcal{E}^s(\boldsymbol{\eta}^{s, t}, \mathbf{X}^s) }{\partial \boldsymbol{\eta}^{s, t}},\\
				\Delta_2 = \ & \frac{\partial \zeta(\mathbf{F}, \mathbf{Z}^s) }{\partial \mathbf{F}} \frac{\partial \boldsymbol{\delta}^s \cdot \mathcal{E}^s(\boldsymbol{\eta}^{s, t}, \mathbf{X}^s) }{\partial \boldsymbol{\eta}^{s, t}}\\
				+ \ & \frac{\partial \zeta(\mathbf{F}, \mathbf{Z}^s) }{\partial \mathbf{Z}^s} \frac{\partial \mathcal{E}^s(\boldsymbol{\eta}^{s, t}, \mathbf{X}^s) }{\partial \boldsymbol{\eta}^{s, t}}.
			\end{split}
		\end{equation}
		
		\newtheorem{theorem}{Theorem}[section]
		\newtheorem{lemma}[theorem]{Lemma}
		
		\begin{theorem}\label{the:theorem2}
			The proposed VCR adaptively modulates the gradients for updating parameters of view-specific encoders, facilitating balanced multi-view clustering.
		\end{theorem}
		\begin{proof}
			Since $\sum_{j=1}^n \hat{g}^s_{ij} = 1,  \hat{g}^s_{ij} > 0$, the Eq.~\eqref{eq:vcr_loss} can be further rewritten into the matrix form as,
			\begin{equation}\label{eq:vcr_loss_re}
				\zeta = \frac{\mathrm{Tr}\Big(((\mathbf{1} - \hat{\mathbf{G}}^s)\odot \mathbf{A}^s)^2 \Big)}{N} + \frac{\mathrm{Tr}\Big((\hat{\mathbf{G}}^s\odot (\mathbf{1} - \mathbf{A}^s))^2\Big) }{N^2 - N} ,
			\end{equation}
			where $\odot$ is the Hadamard product, and $\mathbf{A}^s$ represents the correlation matrix for $s$-th view computed from $\mathbf{C}^s$. Eq.~\eqref{eq:vcr_loss_re} measures the alignment between matrix $\hat{\mathbf{G}}^s$ and $\mathbf{A}^s$. And $\zeta(\cdot)$ achieves the minimum value when $\hat{\mathbf{G}}^s$ and $\mathbf{A}^s$ share similar distributions as well as the clustering structures. 
			
			We assume that view 1 contains more discriminative information than view 2. Consequently, the fused representation $\mathbf{F}$ is expected to offer clearer clustering structures compared to view-specific features $\mathbf{Z}^1$ and $\mathbf{Z}^2$. As illustrated in Section~\ref{sec:imvca}, the view with more discriminative information will dominate the learning process. Thus, the representation $\mathbf{F}$ will contain a similar clustering structure with $\mathbf{Z}^1$, since the clustering knowledge of $\mathbf{F}$ mainly comes from $\mathbf{Z}^1$ when only employ the feature reconstruction loss. Consequently, the resultant $\bar{\mathbf{G}}^s$ and $\mathbf{G}$ both well align with $\mathbf{A}^s$, resulting a small value for $\zeta(\mathbf{F}, \mathbf{Z}^1)$. On the contrary, $\zeta(\mathbf{F}, \mathbf{Z}^2)$ should be a large value due to the fewer contributions from $\mathbf{Z}^2$ in the learning process of the model.
			
			$\zeta(\mathbf{F}, \mathbf{Z}^1) < \zeta(\mathbf{F}, \mathbf{Z}^2)$ suggests that a more discriminative view is better at ensuring that the local similarities captured by the indicators align with the neighborhood structures of both the joint features and the original view-specific features. Consequently, view 2 must learn a clustering distribution that closely resembles the joint features to significantly reduce the loss $\zeta(\mathbf{F}, \mathbf{Z}^2)$. The clustering distribution of the joint features should also be aligned more closely with that of view 2, further facilitating the reduction of the loss $\zeta(\mathbf{F}, \mathbf{Z}^2)$. To implement this process, larger values of$\frac{\partial \zeta(\mathbf{F}, \mathbf{Z}^2) }{\partial \mathbf{F}}$ and 
			$\frac{\partial \zeta(\mathbf{F}, \mathbf{Z}^2) }{\partial \mathbf{Z}^2}$ should be achieved, while smaller values of $\frac{\partial \zeta(\mathbf{F}, \mathbf{Z}^1) }{\partial \mathbf{F}}$ and 
			$\frac{\partial \zeta(\mathbf{F}, \mathbf{Z}^1) }{\partial \mathbf{Z}^1}$ should be maintained. As a result, the VCR assigns a larger scale factor to the gradient of the view with less discriminative information and a lower scale factor to the gradient of the view with more discriminative information, thereby enhancing feature learning for its view-specific encoder. Ultimately, under the guidance of VCR, the model will achieve a more balanced multi-view clustering. This completes the proof.
			
		\end{proof}
		
		Thus, based on the Theorem~\ref{the:theorem2}, Our proposed VCR provides additional gradients for the parameters updating in view-specific encoders and performs balanced learning for multi-view clustering.

		\section{Experiments}
		
		\subsection{Experimental Settings}
		
		\subsubsection{Datasets}
		In our experiments, we evaluate our proposed method on eight widely used multi-view datasets. The details are as follows:
		
		\textbf{CUB}~\footnote{https://www.vision.caltech.edu/visipedia/CUB-200.html}: This dataset consists of 600 images of different bird species, each accompanied by text descriptions, spanning 10 categories. Each sample is represented by 4096-dimensional deep image features and 300-dimensional text features.
		
		\textbf{HW}~\footnote{https://archive.ics.uci.edu/ml/datasets/Multiple+Features}: This dataset contains 2000 handwritten images corresponding to the digits 0-9. Each image is characterized by three views: 76-dimensional FOU features, 216-dimensional FAC features, and 240-dimensional Pix features.
		
		\textbf{Youtube}~\cite{wolf2011face}: This video dataset comprises 2000 samples from 10 classes, with six different types of features extracted to represent each sample.
		
		\textbf{OutdoorScene}~\cite{hu2020multi}: This dataset includes 2688 images of outdoor scenes captured from eight different scene groups. Each image is described using four views: 512-dimensional GIST features, 432-dimensional HOG features, 256-dimensional LBP features, and 48-dimensional Gabor features.
		
		\textbf{RGB-D}~\cite{kong2014you}: This dataset consists of 1449 indoor scene images, each with corresponding text descriptions belonging to 13 categories. Pre-trained ResNet-50 and doc2vec models are utilized to extract 2048-dimensional deep image features and 300-dimensional text features, respectively.
		
		\textbf{nuswide}~\cite{zhen2019deep}: This dataset includes 9000 images along with corresponding tags from 10 categories. Each sample is represented by 4096-dimensional deep image features and 300-dimensional text features.
		
		\textbf{xrmb}~\cite{zhen2019deep}: This dataset consists of 400000 text-image pairs across 200 classes. Each sample is represented by 4096-dimensional deep image features and 300-dimensional text features.
		
		\textbf{xmedia}~\cite{zhen2019deep}: This dataset contains 85297 samples collected from both acoustic and articulation views, spanning 39 classes. Two different types of features are extracted to represent each sample.
		
		For the nuswide, xrmb, xmedia datasets, we randomly selected 3000, 4852, and 5000 samples, respectively, for our experiments. The features in all datasets are scaled into the range $[0, 1]$ in our experiments.
		
		\subsubsection{Compared Methods}
		To verify the superiority of proposed method, we compare it with seven shallow multi-view clustering methods (e.g., \textbf{AWP}~\cite{nie2018multiview}, \textbf{GMC}~\cite{wang2019gmc}, \textbf{LMVSC}~\cite{kang2020large}, \textbf{OPMC}~\cite{liu2021one}, \textbf{EEOMVC}~\cite{wang2023efficient}, \textbf{UDBGL}~\cite{fang2023efficient}, \textbf{CAMVC}~\cite{zhang2024learning}) and seven deep multi-view clustering methods (e.g., \textbf{DCCA}~\cite{wang2015deep}, \textbf{MFLVC}~\cite{xu2022multi}, \textbf{DealMVC}~\cite{yang2023dealmvc}, \textbf{CVCL}~\cite{chen2023deep}, \textbf{SBMvC}~\cite{wu2024self}, \textbf{SCM}~\cite{luo2024simple}, \textbf{MVCAN}~\cite{xu2024investigating}). For all compared approaches, we tune the parameters with a grid search scheme as suggested in their papers to implement their best clustering performance.
		
		\begin{table*}[!htbp]
			\centering
			\caption{The clustering performance measured by ACC, NMI, ARI, and Fscore of all compared methods on eight multi-view datasets. The highest and the second highest values under each metric are {\color[HTML]{EA6B66}\textbf{bolded}} and {\color[HTML]{67AB9F}\ul{ underlined}}, respectively.}
			\label{tab:mvc_results}
			\resizebox{\textwidth}{!}{%
				\begin{tabular}{@{}c|cccc|cccc|cccc|cccc@{}}
					\toprule[1.5pt]
					Datasets                                                                & \multicolumn{4}{c|}{CUB}                                                                                                                                  & \multicolumn{4}{c|}{HW}                                                                                                                                   & \multicolumn{4}{c|}{Youtube}                                                                                                                              & \multicolumn{4}{c}{OutdoorScene}                                                                                                                          \\ \midrule
					Methods                                                                 & ACC                                  & NMI                                  & ARI                                  & Fscore                               & ACC                                  & NMI                                  & ARI                                  & Fscore                               & ACC                                  & NMI                                  & ARI                                  & Fscore                               & ACC                                  & NMI                                  & ARI                                  & Fscore                               \\ \midrule
					AWP{\tiny\textcolor{gray}{[KDD18]}}\cite{nie2018multiview}          & 81.17                                & 75.20                                & 65.52                                & 69.00                                & {\color[HTML]{67AB9F}\ul{95.80}}     & {\color[HTML]{67AB9F}\ul{91.67}}     & {\color[HTML]{67AB9F}\ul{91.11}}     & {\color[HTML]{67AB9F}\ul{91.99}}     & 31.40                                & 18.26                                & 12.13                                & 22.02                                & 60.83                                & 45.51                                & 37.37                                & 45.98                                \\
					GMC{\tiny\textcolor{gray}{[TKDE19]}}\cite{wang2019gmc}         & 79.50                                & 78.95                                & 66.48                                & 70.03                                & 85.20                                & 90.26                                & 82.70                                & 84.56                                & 12.15                                & 5.19                                 & 0.15                                 & 18.05                                & 33.56                                & 42.33                                & 18.78                                & 35.02                                \\
					LMVSC{\tiny\textcolor{gray}{[   AAA20]}}\cite{kang2020large}     & 78.03                                & 74.70                                & 63.80                                & 67.46                                & 82.50                                & 77.52                                & 70.87                                & 73.85                                & 25.01                                & 10.06                                & 5.33                                 & 14.97                                & 54.67                                & 39.20                                & 31.11                                & 39.89                                \\
					OPMC{\tiny\textcolor{gray}{[ICCV21]}}\cite{liu2021one}        & 71.50                                & 75.67                                & 60.92                                & 64.95                                & 90.60                                & 83.33                                & 80.70                                & 82.63                                & 25.65                                & 13.80                                & 8.38                                 & 17.73                                & 63.17                                & 51.57                                & 42.22                                & 49.62                                \\
					EEOMVC{\tiny\textcolor{gray}{[TNNLS23]}}\cite{wang2023efficient}     & 65.00                                & 66.49                                & 51.32                                & 56.43                                & 70.15                                & 63.81                                & 55.79                                & 60.35                                & 26.25                                & 11.80                                & 7.08                                 & 17.27                                & 57.70                                & 40.35                                & 29.64                                & 39.02                                \\
					UDBGL{\tiny\textcolor{gray}{[TNNLS23]}}\cite{fang2023efficient}      & 79.35                                & {\color[HTML]{67AB9F}\ul{79.82}}     & 68.91                                & 72.17                                & 92.14                                & 85.62                                & 83.69                                & 85.32                                & 38.03                                & 22.74                                & 16.39                                & 24.84                                & 69.46                                & 52.80                                & 45.79                                & 52.77                                \\
					CAMVC{\tiny\textcolor{gray}{[AAAI24]}}\cite{zhang2024learning}       & 81.66                                & 79.67                                & 69.65                                & 72.73                                & 91.24                                & 85.14                                & 82.31                                & 84.08                                & 25.70                                & 11.58                                & 6.60                                 & 16.13                                & 71.16                                & 55.82                                & 48.31                                & 54.91                                \\ \midrule
					DCCA{\tiny\textcolor{gray}{[ICML15]}}\cite{wang2015deep}        & 60.09                                & 55.87                                & 40.79                                & 48.84                                & 85.04                                & 85.15                                & 77.75                                & 83.50                                & 10.68                                & 1.48                                 & 0.01                                 & 19.07                                & 48.16                                & 51.58                                & 36.42                                & 49.76                                \\
					MFLVC{\tiny\textcolor{gray}{[CVPR22]}}\cite{xu2022multi}       & 67.00                                & 62.91                                & 49.24                                & 55.59                                & 82.65                                & 80.41                                & 74.02                                & 77.14                                & 39.70                                & 28.84                                & 20.02                                & 29.36                                & 64.62                                & 55.15                                & 42.90                                & 51.11                                \\
					DealMVC{\tiny\textcolor{gray}{[ACM   MM23]}}\cite{yang2023dealmvc} & 61.00                                & 66.10                                & 51.27                                & 58.04                                & 82.10                                & 78.73                                & 72.57                                & 75.42                                & 36.60                                & 23.98                                & 15.88                                & 26.24                                & 75.26                                & 62.73                                & 54.21                                & 61.83                                \\
					CVCL{\tiny\textcolor{gray}{[ICCV23]}}\cite{chen2023deep}     & 79.33                                & 71.03                                & 61.54                                & 66.17                                & 92.95                                & 87.53                                & 85.43                                & 87.15                                & {\color[HTML]{67AB9F}\ul{46.15}}     & {\color[HTML]{67AB9F}\ul{31.68}}     & {\color[HTML]{67AB9F}\ul{23.32}}     & {\color[HTML]{67AB9F}\ul{31.50}}     & 71.65                                & 60.23                                & 50.34                                & 57.98                                \\
					SBMvC{\tiny\textcolor{gray}{[TMM24]}}\cite{luo2024simple}        & 72.83                                & 68.71                                & 55.82                                & 61.66                                & 90.95                                & 84.76                                & 81.43                                & 83.85                                & 41.60                                & 29.72                                & 20.27                                & 29.08                                & 72.10                                & 60.57                                & 52.56                                & 59.89                                \\
					SCM{\tiny\textcolor{gray}{[IJCAI24]}}\cite{luo2024simple}     & {\color[HTML]{67AB9F}\ul{84.83}}     & 79.11                                & {\color[HTML]{67AB9F}\ul{71.23}}     & {\color[HTML]{67AB9F}\ul{75.17}}     & 89.25                                & 80.57                                & 77.63                                & 80.72                                & 37.20                                & 23.01                                & 16.45                                & 24.73                                & 60.45                                & 51.95                                & 41.59                                & 49.93                                \\
					MVCAN{\tiny\textcolor{gray}{[CVPR24]}}\cite{xu2024investigating}       & 80.67                                & 78.42                                & 67.77                                & 72.99                                & 95.75                                & 91.58                                & 90.95                                & 91.98                                & 29.75                                & 22.31                                & 13.64                                & 25.35                                & {\color[HTML]{67AB9F}\ul{75.89}}     & {\color[HTML]{67AB9F}\ul{63.17}}     & {\color[HTML]{67AB9F}\ul{56.01}}     & {\color[HTML]{67AB9F}\ul{62.50}}     \\ \midrule
					\rowcolor[HTML]{EFEFEF} 
					Ours                                                                    & \color[HTML]{EA6B66}\textbf{{86.67}} & \color[HTML]{EA6B66}\textbf{{82.21}} & \color[HTML]{EA6B66}\textbf{{75.53}} & \color[HTML]{EA6B66}\textbf{{78.44}} & \color[HTML]{EA6B66}\textbf{{97.95}} & \color[HTML]{EA6B66}\textbf{{95.18}} & \color[HTML]{EA6B66}\textbf{{95.46}} & \color[HTML]{EA6B66}\textbf{{96.01}} & \color[HTML]{EA6B66}\textbf{{47.67}} & \color[HTML]{EA6B66}\textbf{{34.25}} & \color[HTML]{EA6B66}\textbf{{26.23}} & \color[HTML]{EA6B66}\textbf{{34.70}} & \color[HTML]{EA6B66}\textbf{{78.46}} & \color[HTML]{EA6B66}\textbf{{64.46}} & \color[HTML]{EA6B66}\textbf{{58.64}} & \color[HTML]{EA6B66}\textbf{{65.12}} \\ \bottomrule
					
					\addlinespace[0.3em]
					
					\toprule
					Datasets                                                                & \multicolumn{4}{c|}{RGB-D}                                                                                                                                                                                                                              & \multicolumn{4}{c|}{nuswide}                                                                                                                                                                                                                               & \multicolumn{4}{c|}{xrmb}                                                                                                                                                                                                                          & \multicolumn{4}{c}{xmedia}                                                                                                                                                                                                                      \\ \midrule
					Methods                                                                 & ACC                                                         & NMI                                                         & ARI                                                         & Fscore                                                      & ACC                                                         & NMI                                                         & ARI                                                         & Fscore                                                      & ACC                                                         & NMI                                                         & ARI                                                         & Fscore                                                      & ACC                                                         & NMI                                                         & ARI                                                         & Fscore                                                      \\ \midrule
					AWP{\tiny\textcolor{gray}{[KDD18]}}\cite{nie2018multiview}          & 41.34                                                       & 28.95                                                       & 18.55                                                       & 27.03                                                       & 55.93                                                       & 47.05                                                       & 36.38                                                       & 44.37                                                       & 24.07                                                       & 38.19                                                       & 13.66                                                       & 16.26                                                       & 81.12                                                       & 84.86                                                       & 74.32                                                       & 74.95                                                       \\
					GMC{\tiny\textcolor{gray}{[TKDE19]}}\cite{wang2019gmc}         & 42.24                                                       & 30.97                                                       & 8.81                                                        & 29.88                                                       & 22.83                                                       & 19.60                                                       & 1.72                                                        & 19.68                                                       & 13.09                                                       & 23.20                                                       & 0.41                                                        & 6.53                                                        & 80.66                                                       & 89.34                                                       & 54.23                                                       & 55.96                                                       \\
					LMVSC{\tiny\textcolor{gray}{[   AAA20]}}\cite{kang2020large}     & {\color[HTML]{67AB9F}\ul{45.50}}                            & {\color[HTML]{67AB9F}\ul{39.23}}                            & 24.94                                                       & 32.83                                                       & 57.02                                                       & 42.89                                                       & 34.93                                                       & 41.65                                                       & 20.89                                                       & 34.13                                                       & 10.75                                                       & 13.34                                                       & 78.45                                                       & 90.57                                                       & 76.47                                                       & 77.09                                                       \\
					OPMC{\tiny\textcolor{gray}{[ICCV21]}}\cite{liu2021one}        & 41.20                                                       & 37.36                                                       & 22.96                                                       & 31.13                                                       & 58.60                                                       & 44.87                                                       & 37.35                                                       & 43.82                                                       & 27.72                                                       & 43.15                                                       & 16.39                                                       & 18.85                                                       & 85.24                                                       & 93.46                                                       & 85.46                                                       & 85.83                                                       \\
					EEOMVC{\tiny\textcolor{gray}{[TNNLS23]}}\cite{wang2023efficient}     & 43.55                                                       & 36.70                                                       & 22.85                                                       & 32.18                                                       & 61.27                                                       & 45.02                                                       & 39.80                                                       & 46.09                                                       & 26.88                                                       & 39.16                                                       & 14.79                                                       & 17.27                                                       & {\color[HTML]{67AB9F}\ul{89.20}}                            & 92.53                                                       & {\color[HTML]{67AB9F}\ul{87.61}}                            & {\color[HTML]{67AB9F}\ul{87.93}}                            \\
					UDBGL{\tiny\textcolor{gray}{[TNNLS23]}}\cite{fang2023efficient}      & 43.69                                                       & 38.36                                                       & {\color[HTML]{67AB9F}\ul{25.71}}                            & 33.32                                                       & 64.49                                                       & 52.38                                                       & 45.59                                                       & 51.25                                                       & 29.36                                                       & 44.05                                                       & 17.63                                                       & 20.08                                                       & 85.74                                                       & {\color[HTML]{67AB9F}\ul{93.80}}                            & 86.63                                                       & 86.98                                                       \\
					CAMVC{\tiny\textcolor{gray}{[AAAI24]}}\cite{zhang2024learning}       & 35.14                                                       & 29.08                                                       & 16.89                                                       & 25.33                                                       & {\color[HTML]{FF0000} \color[HTML]{EA6B66}\textbf{{68.05}}} & {\color[HTML]{FF0000} \color[HTML]{EA6B66}\textbf{{54.88}}} & {\color[HTML]{FF0000} \color[HTML]{EA6B66}\textbf{{51.19}}} & {\color[HTML]{FF0000} \color[HTML]{EA6B66}\textbf{{56.20}}} & 29.00                                                       & 43.86                                                       & 17.93                                                       & 20.35                                                       & 85.36                                                       & 92.47                                                       & 85.88                                                       & 86.24                                                       \\ \midrule
					DCCA{\tiny\textcolor{gray}{[ICML15]}}\cite{wang2015deep}        & 17.19                                                       & 6.51                                                        & 2.32                                                        & 13.58                                                       & 48.16                                                       & 40.03                                                       & 22.89                                                       & 41.40                                                       & 23.15                                                       & 37.09                                                       & 10.88                                                       & 19.42                                                       & 49.68                                                       & 54.96                                                       & 31.46                                                       & 40.99                                                       \\
					MFLVC{\tiny\textcolor{gray}{[CVPR22]}}\cite{xu2022multi}       & 38.03                                                       & 21.41                                                       & 18.37                                                       & 29.24                                                       & 32.10                                                       & 20.75                                                       & 16.09                                                       & 24.40                                                       & 24.69                                                       & 32.82                                                       & 13.17                                                       & 17.52                                                       & 13.28                                                       & 26.71                                                       & 6.93                                                        & 12.33                                                       \\
					DealMVC{\tiny\textcolor{gray}{[ACM   MM23]}}\cite{yang2023dealmvc} & 45.27                                                       & 27.05                                                       & 23.14                                                       & {\color[HTML]{67AB9F}\ul{37.64}}                            & 51.93                                                       & 37.88                                                       & 32.75                                                       & 38.81                                                       & 17.19                                                       & 27.59                                                       & 8.83                                                        & 14.61                                                       & 10.20                                                       & 22.72                                                       & 5.86                                                        & 10.10                                                       \\
					CVCL{\tiny\textcolor{gray}{[   ICCV23]}}\cite{chen2023deep}     & 25.86                                                       & 17.05                                                       & 8.02                                                        & 21.88                                                       & 58.63                                                       & 48.30                                                       & 39.92                                                       & 47.41                                                       & 27.10                                                       & 38.92                                                       & 14.44                                                       & 18.32                                                       & 34.02                                                       & 51.15                                                       & 24.53                                                       & 30.46                                                       \\
					SBMvC{\tiny\textcolor{gray}{[TMM24]}}\cite{wu2024self}        & 31.06                                                       & 23.25                                                       & 15.50                                                       & 25.28                                                       & 53.97                                                       & 40.63                                                       & 34.91                                                       & 41.37                                                       & 29.78                                                       & 43.76                                                       & 18.02                                                       & 21.84                                                       & 33.18                                                       & 48.07                                                       & 26.32                                                       & 26.09                                                       \\
					SCM{\tiny\textcolor{gray}{[   IJCAI24]}}\cite{luo2024simple}     & 35.20                                                       & 30.35                                                       & 17.84                                                       & 29.31                                                       & 57.63                                                       & 44.17                                                       & 38.51                                                       & 44.56                                                       & 29.55                                                       & 41.92                                                       & 16.52                                                       & 20.14                                                       & 85.96                                                       & 90.71                                                       & 83.96                                                       & 84.19                                                       \\
					MVCAN{\tiny\textcolor{gray}{[CVPR24]}}\cite{xu2024investigating}       & 40.99                                                       & 33.71                                                       & 22.33                                                       & 32.82                                                       & 62.93                                                       & 49.58                                                       & 44.32                                                       & 49.86                                                       & {\color[HTML]{FF0000} \color[HTML]{EA6B66}\textbf{{32.11}}} & {\color[HTML]{67AB9F}\ul{45.99}}                            & {\color[HTML]{67AB9F}\ul{18.95}}                            & {\color[HTML]{67AB9F}\ul{23.05}}                            & 86.46                                                       & 93.13                                                       & 86.65                                                       & 87.43                                                       \\ \midrule
					\rowcolor[HTML]{EFEFEF} 
					{\color[HTML]{333333} Ours}                                             & {\color[HTML]{333333} \color[HTML]{EA6B66}\textbf{{45.84}}} & {\color[HTML]{333333} \color[HTML]{EA6B66}\textbf{{39.66}}} & {\color[HTML]{333333} \color[HTML]{EA6B66}\textbf{{27.34}}} & {\color[HTML]{333333} \color[HTML]{EA6B66}\textbf{{38.03}}} & {\color[HTML]{333333} {\color[HTML]{67AB9F}\ul{66.49}}}     & {\color[HTML]{333333} {\color[HTML]{67AB9F}\ul{53.94}}}     & {\color[HTML]{333333} {\color[HTML]{67AB9F}\ul{49.53}}}     & {\color[HTML]{333333} {\color[HTML]{67AB9F}\ul{54.03}}}     & {\color[HTML]{333333} {\color[HTML]{67AB9F}\ul{31.03}}}     & {\color[HTML]{333333} \color[HTML]{EA6B66}\textbf{{46.56}}} & {\color[HTML]{333333} \color[HTML]{EA6B66}\textbf{{19.33}}} & {\color[HTML]{333333} \color[HTML]{EA6B66}\textbf{{23.22}}} & {\color[HTML]{333333} \color[HTML]{EA6B66}\textbf{{90.43}}} & {\color[HTML]{333333} \color[HTML]{EA6B66}\textbf{{94.48}}} & {\color[HTML]{333333} \color[HTML]{EA6B66}\textbf{{89.20}}} & {\color[HTML]{333333} \color[HTML]{EA6B66}\textbf{{90.21}}} \\ \bottomrule[1.5pt]
				\end{tabular}%
			}
		\end{table*}
		
		\subsubsection{Evaluation Metrics}
		We employ four evaluation metrics, including accuracy (ACC), normalized mutual information (NMI), adjusted Rand index (ARI), and Fscore to measure the clustering performance of different methods. Note that a higher value of each metric indicates a better clustering result.
		
		\subsubsection{Implementation Details}
		In our proposed method, the view-specific encoders and decoders are all implemented in the fully connected architectures. The detailed structure of the encoders and decoders are $D$-196-128-64 and 64-128-196-$D$, respectively, where $D$ is the input and output dimensions of encoders and decoders. The proposed method is implemented via the PyTorch~\cite{paszke2019pytorch} tool and conducts experiments on a single NVIDIA 2080Ti GPU with a Ubuntu 20.04 platform. The Adam optimizer~\cite{kingma2014adam} is utilized to optimize the proposed method, and the initial learning rate is set to 0.001. The model is trained for 3000 epochs on all datasets, and the obtained joint feature representation is used to obtain the clustering results via the K-means. Besides, the trade-off parameter $\lambda$ is searched in range $[10^{-5}, 10^{-4}, ...,10^4, 10^5]$  with a grid search scheme.
		
		\subsection{Clustering Performance Evaluation}
		The clustering results measured by four metrics on eight multi-view datasets of all compared methods are reported in Tab.~\ref{tab:mvc_results}. From this table, we obtain the following observations:
		\begin{itemize}
			\item The proposed BMvC method consistently performs better than other compared methods on most datasets. For instance, BMvC achieves 78.46, 64.46, 58.64, and 65.12 percentages in terms of ACC, NMI, ARI, and Fscore on the OutdoorScene dataset. Such results exceed the second performer (MVCAN) by about 2.57, 1.29, 2.63, and 2.62 percentages measured by ACC, NMI, ARI, and Fscore on the OutdoorScene dataset. Such clustering performance benefits strongly demonstrate the effectiveness and superiority of the proposed BMvC method, which facilitates more balanced multi-view learning through the introduction of view-specific contrastive regularization.
			\item The proposed BMvC method is superior to other deep-based multi-view clustering competitors, which seldom takes the imbalanced learning problem in multi-view clustering into consideration. The MVCAN method assigns different views with weights based on the mutual information between joint clustering indicators and view-specific ones to mitigate the side effects of less robust views. However, this approach may further exacerbate the issues of imbalanced learning, as the information in certain views with lower weight allocations is more difficult to learn.
			\item Compared with other shallow multi-view clustering methods, our BMvC method achieves superior clustering results. We note that the graph-based methods, e.g., AWP, GMC, and UDBGL usually obtain considerable performance. This indicates that effectively leveraging the similarities among different samples is beneficial for clustering tasks.
		\end{itemize}

		\begin{table*}[!htbp]
			\centering
			\caption{The ablation study of the proposed method in terms of ACC, NMI, ARI, and Fscore on eight benchmark multi-view datasets.}
			\label{tab:mvc_ab}
			\resizebox{\textwidth}{!}{%
				\begin{tabular}{@{}ccccccccccccccccc@{}}
					\toprule[1.5pt]
					\multicolumn{1}{c|}{Datasets}    & \multicolumn{4}{c|}{CUB}                      & \multicolumn{4}{c|}{HW}                       & \multicolumn{4}{c|}{Youtube}                  & \multicolumn{4}{c}{OutdoorScene} \\ \midrule
					\multicolumn{1}{c|}{Methods}     & ACC   & NMI   & ARI   & \multicolumn{1}{c|}{Fscore} & ACC   & NMI   & ARI   & \multicolumn{1}{c|}{Fscore} & ACC   & NMI   & ARI   & \multicolumn{1}{c|}{Fscore} & ACC     & NMI     & ARI     & Fscore   \\ \midrule
					\multicolumn{17}{l}{(a) Effectiveness of View-specific Contrasitve Regularization}                                                                                                                                                          \\ \midrule
					\multicolumn{1}{c|}{Rec}    & 78.54 & 77.43 & 65.44 & \multicolumn{1}{c|}{71.99}  & 90.77 & 83.89 & 81.40 & \multicolumn{1}{c|}{83.55}  & 27.48 & 21.29 & 10.93 & \multicolumn{1}{c|}{24.80}  & 72.06   & 56.98   & 49.12   & 56.95    \\
					\multicolumn{1}{c|}{VCR}    & 83.81 & 80.70 & 72.86 & \multicolumn{1}{c|}{76.10}  & 97.34 & 94.25 & 94.13 & \multicolumn{1}{c|}{94.87}  & 42.91 & 30.84 & 22.81 & \multicolumn{1}{c|}{32.15}  & 73.49   & 60.08   & 52.36   & 59.33    \\
					\multicolumn{1}{c|}{Rec + VCR}   & 86.67 & 82.21 & 75.53 & \multicolumn{1}{c|}{78.44}  & 97.95 & 95.18 & 95.46 & \multicolumn{1}{c|}{96.01}  & 47.67 & 34.25 & 26.23 & \multicolumn{1}{c|}{34.70}  & 78.46   & 64.46   & 58.64   & 65.12    \\ \midrule
					\multicolumn{17}{l}{(b) Impats of Diverse Feature Fusion manners}                                                                                                                                                                            \\ \midrule
					\multicolumn{1}{c|}{BMvC-w-ASum} & 89.30 & 85.57 & 80.19 & \multicolumn{1}{c|}{82.49}  & 97.47 & 94.41 & 94.39 & \multicolumn{1}{c|}{95.13}  & 47.01 & 32.71 & 25.14 & \multicolumn{1}{c|}{33.43}  & 75.87   & 62.23   & 55.77   & 62.15    \\
					\multicolumn{1}{c|}{BMvC-w-WSum} & 88.32 & 84.90 & 79.07 & \multicolumn{1}{c|}{81.52}  & 97.51 & 94.62 & 94.50 & \multicolumn{1}{c|}{95.22}  & 46.14 & 32.09 & 24.88 & \multicolumn{1}{c|}{32.91}  & 78.34   & 64.01   & 58.49   & 64.75    \\
					\multicolumn{1}{c|}{Ours-w-Cat}  & 86.67 & 82.21 & 75.53 & \multicolumn{1}{c|}{78.44}  & 97.95 & 95.18 & 95.46 & \multicolumn{1}{c|}{96.01}  & 47.67 & 34.25 & 26.23 & \multicolumn{1}{c|}{34.70}  & 78.46   & 64.46   & 58.64   & 65.12    \\ \bottomrule
					
					\addlinespace[0.3em]
					
					\toprule
					\multicolumn{1}{c|}{Datasets}    & \multicolumn{4}{c|}{RGB-D}                                                                                & \multicolumn{4}{c|}{nuswide}                                                                      & \multicolumn{4}{c|}{xrmb}                                                                         & \multicolumn{4}{c}{xmedia}                                                                       \\ \midrule
					\multicolumn{1}{c|}{Methods}     & \multicolumn{1}{c}{ACC} & \multicolumn{1}{c}{NMI} & \multicolumn{1}{c}{ARI} & \multicolumn{1}{c|}{Fscore} & \multicolumn{1}{c}{ACC} & \multicolumn{1}{c}{NMI} & \multicolumn{1}{c}{ARI} & \multicolumn{1}{c|}{Fscore} & \multicolumn{1}{c}{ACC} & \multicolumn{1}{c}{NMI} & \multicolumn{1}{c}{ARI} & \multicolumn{1}{c|}{Fscore} & \multicolumn{1}{c}{ACC} & \multicolumn{1}{c}{NMI} & \multicolumn{1}{c}{ARI} & \multicolumn{1}{c}{Fscore} \\ \midrule
					\multicolumn{17}{l}{(a) Effectiveness   of View-specific Contrasitve Regularization}                                                                                                                                                                                                                                                                                                                                                                                            \\ \midrule
					\multicolumn{1}{c|}{Rec}    & 38.34                   & 33.91                   & 20.99                   & \multicolumn{1}{l|}{31.67}  & 59.63                   & 47.69                   & 39.75                   & \multicolumn{1}{l|}{46.95}  & 29.30                   & 44.17                   & 17.25                   & \multicolumn{1}{l|}{21.45}  & 85.21                   & 93.39                   & 85.94                   & 87.35                      \\
					\multicolumn{1}{c|}{VCR}    & 43.01                   & 40.61                   & 26.07                   & \multicolumn{1}{l|}{37.50}  & 68.06                   & 54.49                   & 50.78                   & \multicolumn{1}{l|}{55.16}  & 28.19                   & 44.99                   & 17.71                   & \multicolumn{1}{l|}{21.70}  & 87.74                   & 94.75                   & 88.40                   & 89.69                      \\
					\multicolumn{1}{c|}{Rec + VCR}   & 45.84                   & 39.66                   & 27.34                   & \multicolumn{1}{l|}{38.03}  & 66.49                   & 53.94                   & 49.53                   & \multicolumn{1}{l|}{54.03}  & 31.03                   & 46.56                   & 19.33                   & \multicolumn{1}{l|}{23.22}  & 90.43                   & 94.48                   & 89.20                   & 90.21                      \\ \midrule
					\multicolumn{17}{l}{(b) Impats of   Diverse Feature Fusion manners}                                                                                                                                                                                                                                                                                                                                                                                                              \\ \midrule
					\multicolumn{1}{c|}{BMvC-w-ASum} & 45.28                   & 38.98                   & 26.85                   & \multicolumn{1}{l|}{37.87}  & 67.44                   & 54.19                   & 49.91                   & \multicolumn{1}{l|}{54.61}  & 30.79                   & 46.32                   & 19.17                   & \multicolumn{1}{l|}{23.26}  & 90.97                   & 95.17                   & 90.16                   & 91.17                      \\
					\multicolumn{1}{c|}{BMvC-w-WSum} & 44.37                   & 41.98                   & 27.37                   & \multicolumn{1}{l|}{39.03}  & 66.57                   & 53.91                   & 48.21                   & \multicolumn{1}{l|}{54.08}  & 30.71                   & 45.74                   & 18.86                   & \multicolumn{1}{l|}{22.78}  & 87.15                   & 94.24                   & 87.91                   & 88.99                      \\
					\multicolumn{1}{c|}{Ours-w-Cat}  & 45.84                   & 39.66                   & 27.34                   & \multicolumn{1}{l|}{38.03}  & 66.49                   & 53.94                   & 49.53                   & \multicolumn{1}{l|}{54.03}  & 31.03                   & 46.56                   & 19.33                   & \multicolumn{1}{l|}{23.22}  & 90.43                   & 94.48                   & 89.20                   & 90.21                      \\ \bottomrule[1.5pt]
					
				\end{tabular}%
			}
		\end{table*}

		\begin{figure*}[!htbp]
			\centering
			\subfigure[CUB]{
				\includegraphics[width=0.23\textwidth]{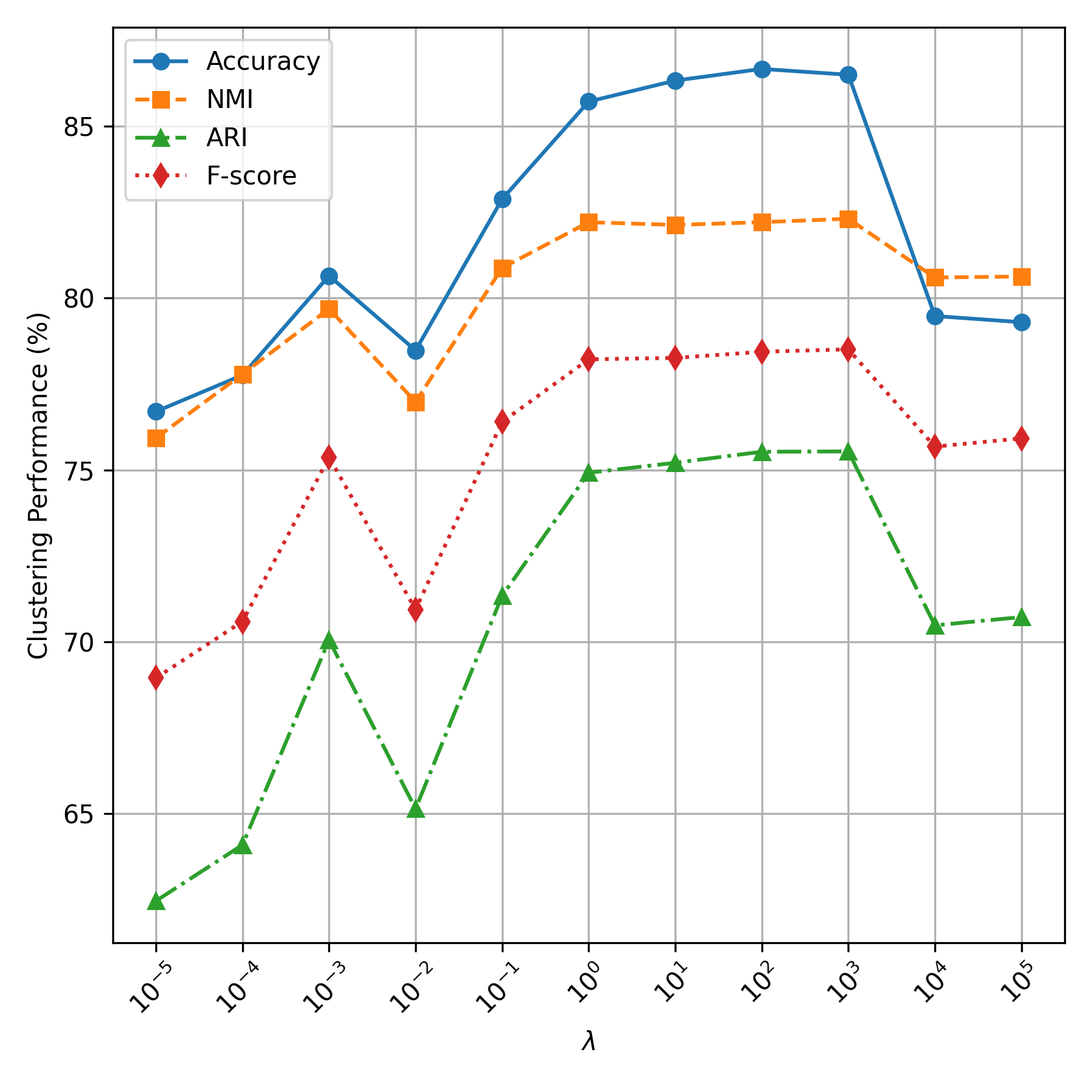}\label{fig:CUB}}
			\subfigure[HW]{
				\includegraphics[width=0.23\textwidth]{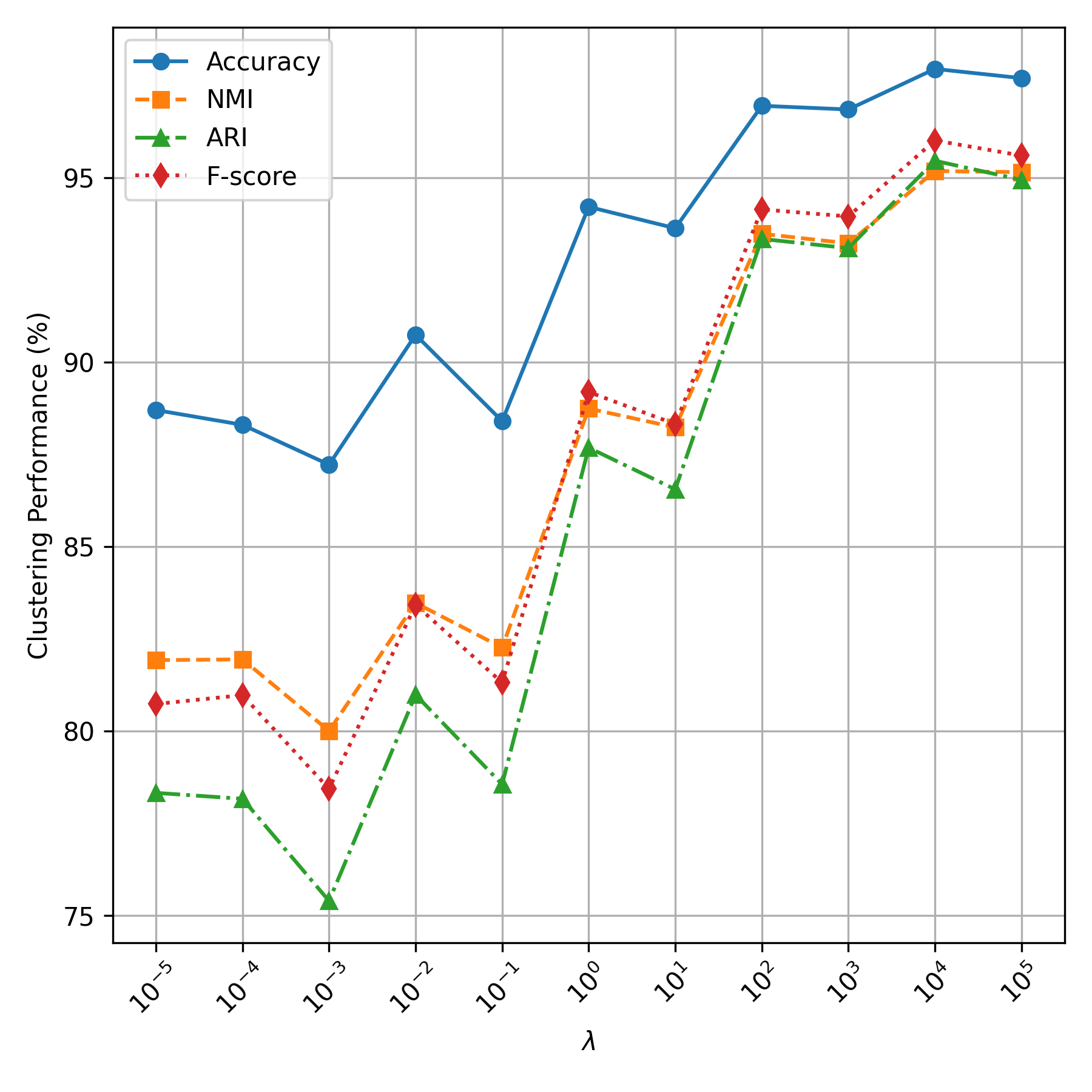}\label{fig:HW}}
			\subfigure[Youtube]{
				\includegraphics[width=0.23\textwidth]{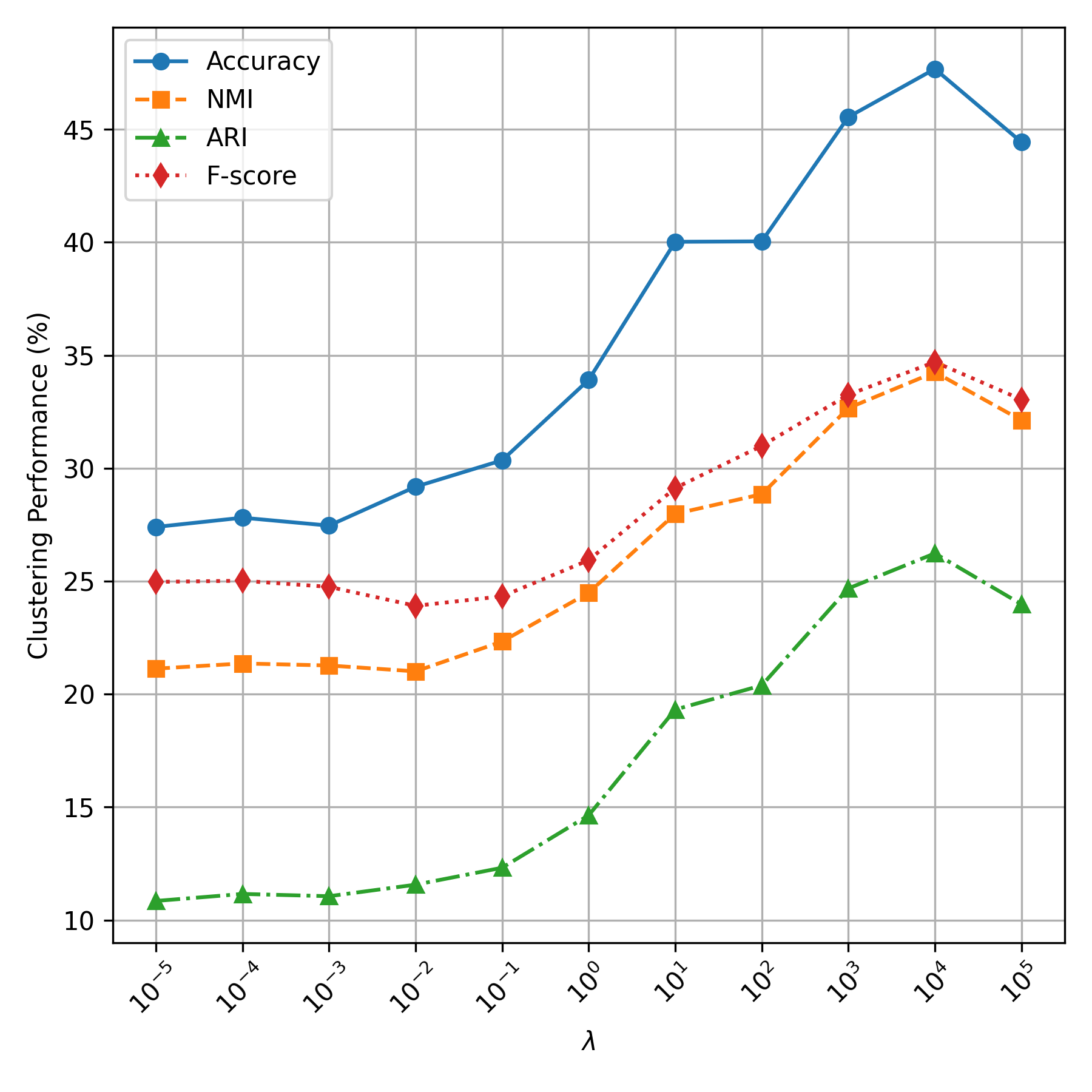}\label{fig:Youtube}}
			\subfigure[ORL]{
				\includegraphics[width=0.23\textwidth]{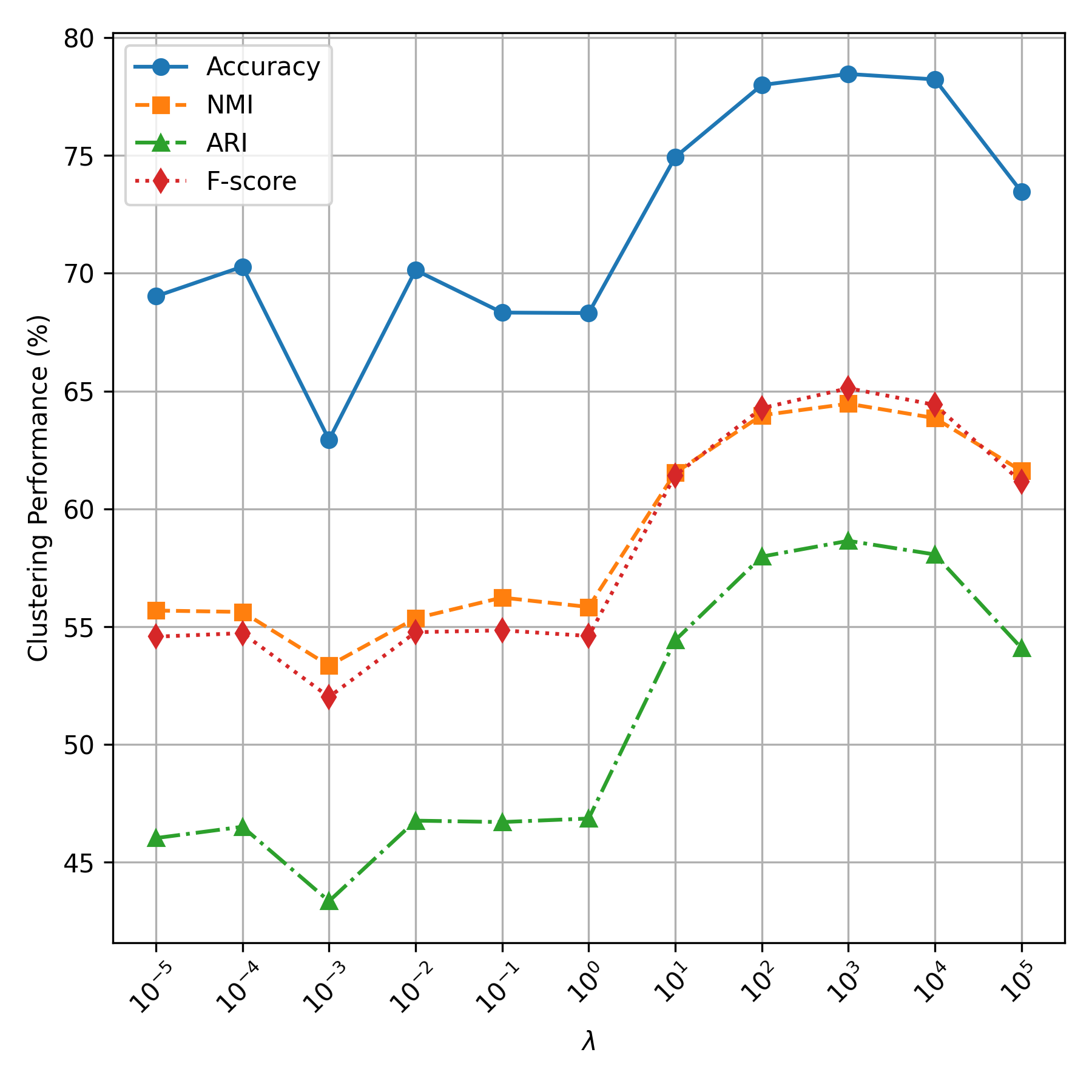}\label{fig:ORL}}
			\subfigure[RGB-D]{
				\includegraphics[width=0.23\textwidth]{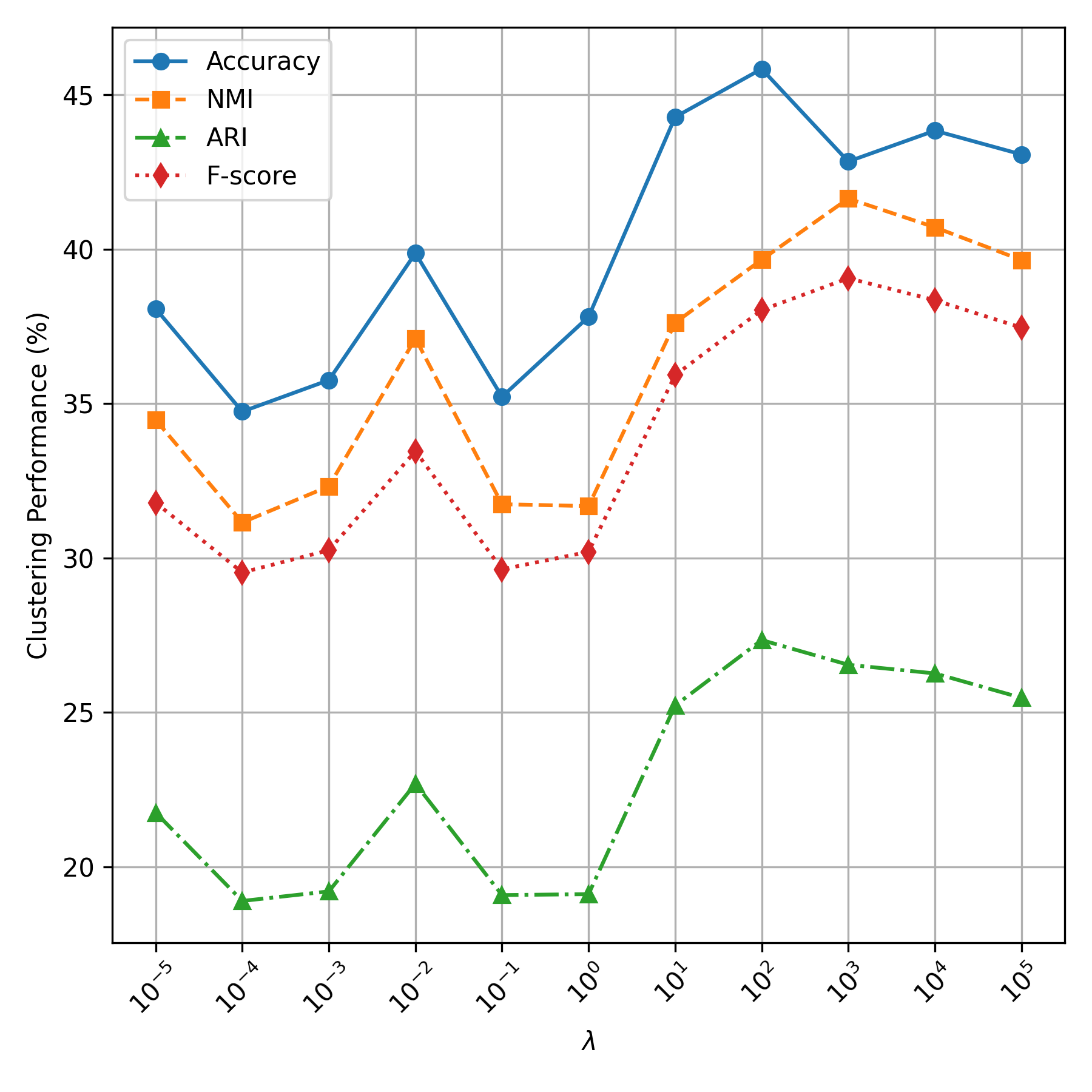}\label{fig:RGB-D}}
			\subfigure[nuswide]{
				\includegraphics[width=0.23\textwidth]{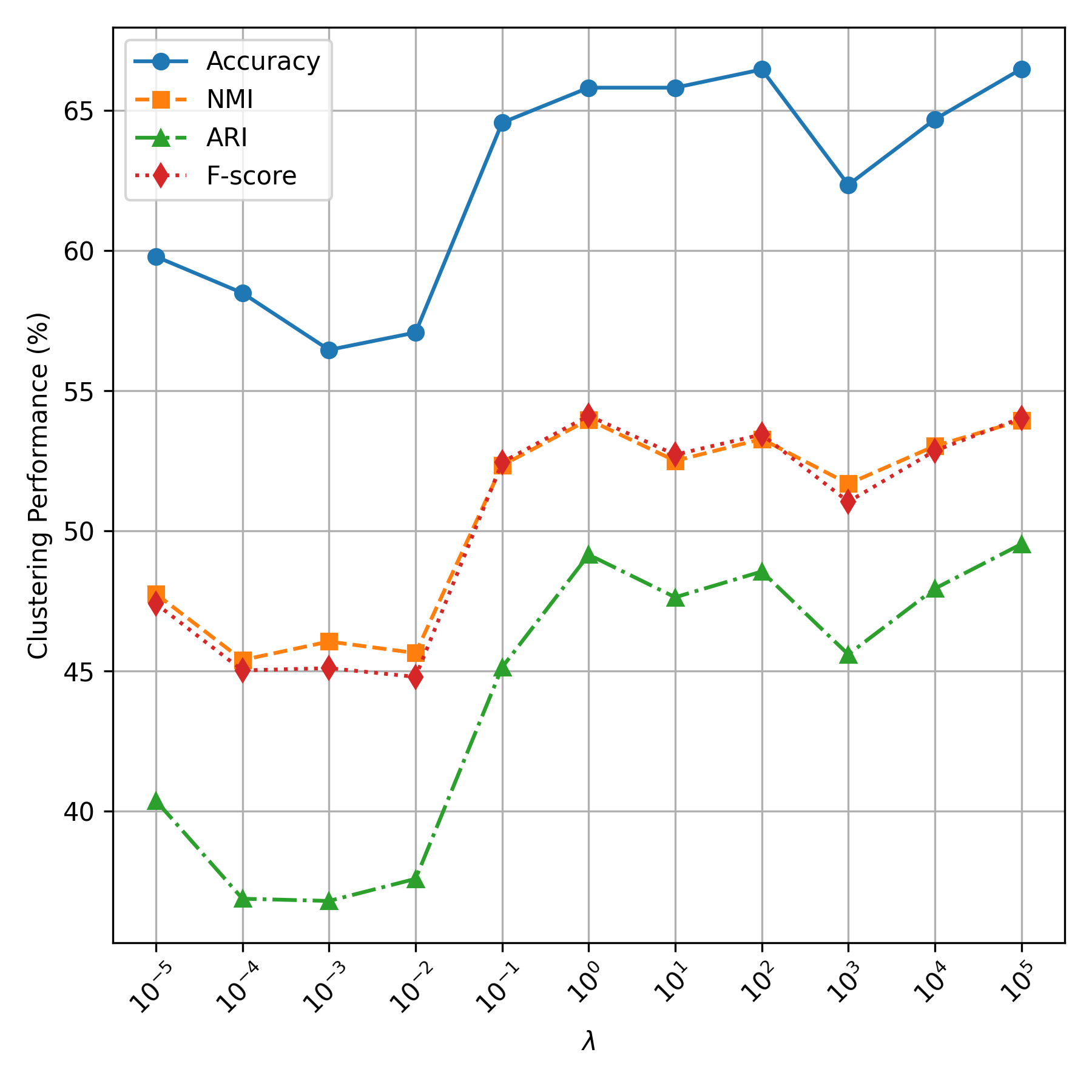}\label{fig:nuswide}}
			\subfigure[xrmb]{
				\includegraphics[width=0.23\textwidth]{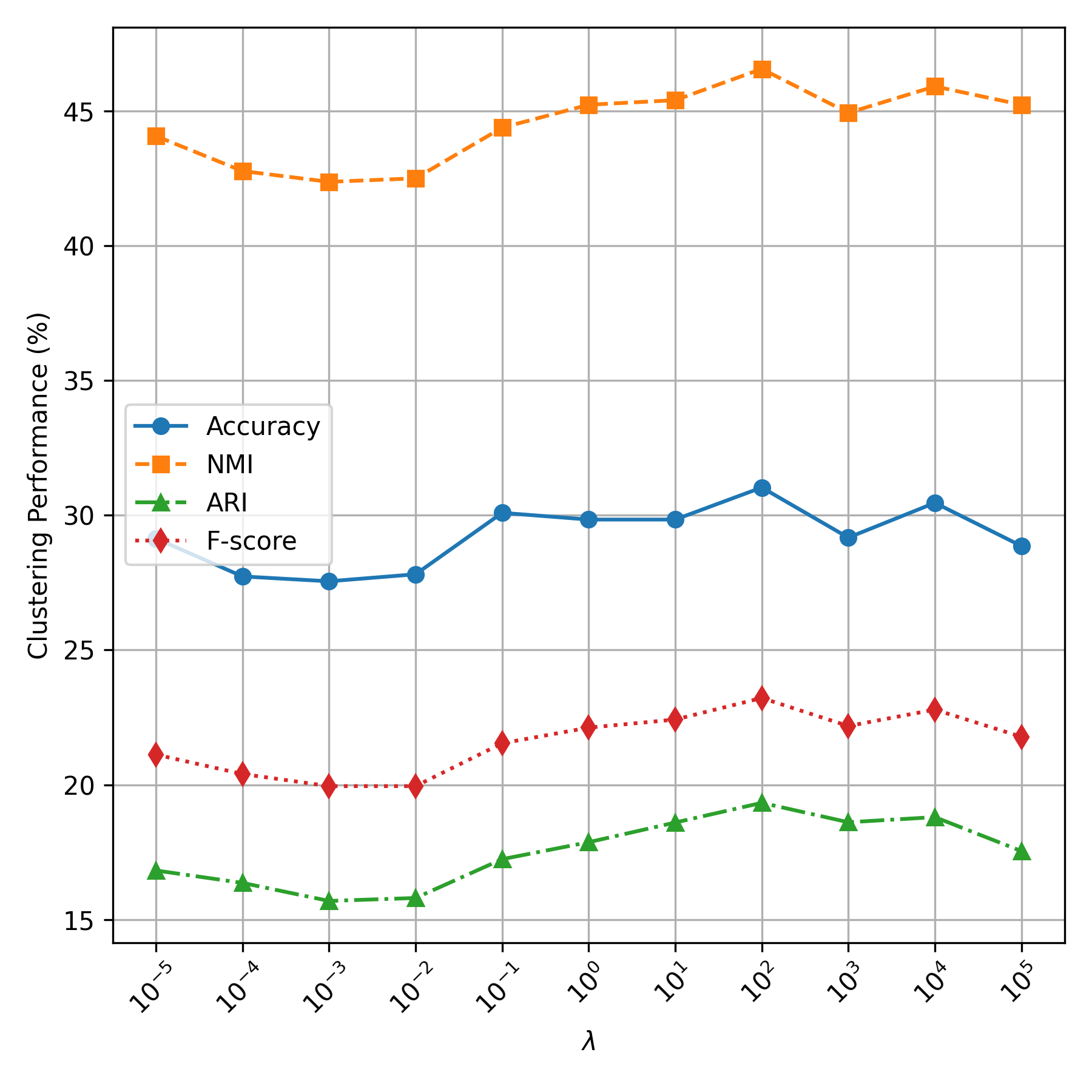}\label{fig:xrmb}}
			\subfigure[xmedia]{
				\includegraphics[width=0.23\textwidth]{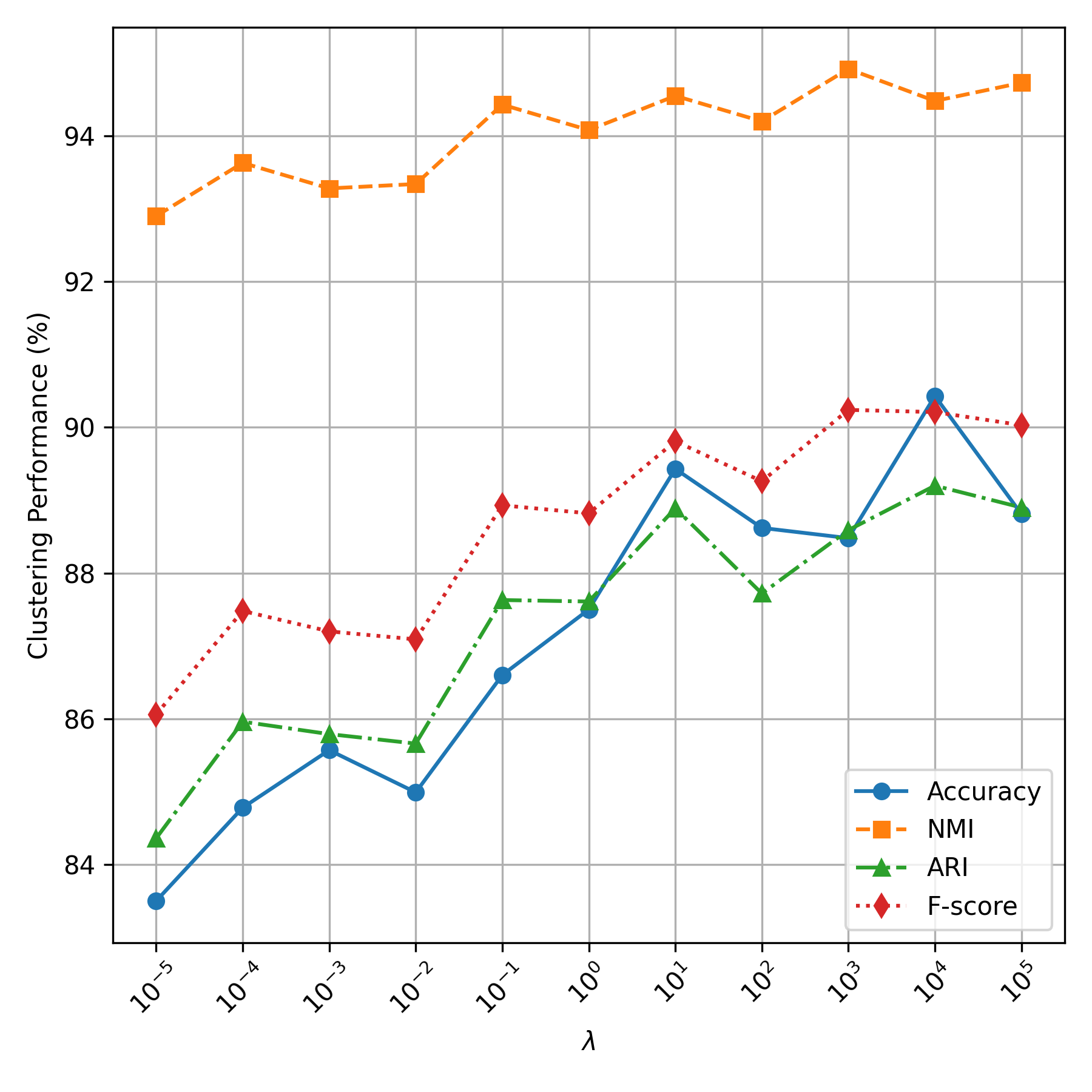}\label{fig:xmedia}}
			\caption{The parameter sensitivity of the proposed method on eight benchmark multi-view datasets in terms of ACC, NMI, ARI, and Fscore, respectively.}
			\label{fig:ps}
		\end{figure*}
		
		\subsection{Ablation Studies}
		In this section, we conduct a series of experiments on the eight multi-view datasets to analyze the effectiveness of diverse components in the proposed method. 
		
		\subsubsection{Effectiveness of View-specific Contrastive Regularization}
		The view-specific contrastive regularization, which leverages the similarities captured from original features and joint features to perform contrastive learning on the view-specific clustering space, ensuring additional gradients for the parameters updating in each view-specific encoder, is the key component of the proposed method. To verify its effectiveness, we formulate three methods here, i.e., Rec, VCR, and Rec + VCR. For the Rec method, the view-specific contrastive regularization is removed from the BMvC method and only the feature reconstruction loss is utilized to optimize the model. VCR method cut out the view-specific decoder parts and just leverages the view-specific contrastive regularization loss to optimize the model. Rec+VCR is set as same as our proposed BMvC methods, in which the feature reconstruction and view-specific contrastive regularization are jointly exploited. The clustering performance in terms of ACC, NMI, ARI, and Fscore on eight multi-view datasets of the above formulated methods is reported in Tab.~\ref{tab:mvc_ab}. From the results, we find that 1) The VCR method performs better than Rec in most cases, which indicates the VCR part can obtain more balanced multi-view learning during training, ensuring that the complementary information of multi-view datasets can be more explored. 2) Rec + VCR is consistently superior to Rec and VCR, demonstrating jointly utilizing Rec and VCR is the best choice for multi-view clustering.
		
		\subsubsection{Impacts of Different Feature Fusion manners}
		Different multi-view feature fusion manners may achieve diverse impacts on the final multi-view clustering model. To study this, we construct three methods, i.e., BMvC-w-Asum, BMvC-w-Wsum, and BMvC-w-Cat, which fuse multi-view features via average feature addition, feature addition with sample-wise weights, and feature concatenation, respectively. As can be seen from the results in Tab.~\ref{tab:mvc_ab}, our BMvC method equipped with average feature addition, feature addition with sample-wise weights, and feature concatenation can both obtain considerable clustering performance. This indicates that our proposed BMvC method is less sensitive to different multi-view feature fusion manners.
		
		\subsubsection{Parameter Sensitivity Analysis}
		Our proposed BMvC method consists of a key balance parameter $\lambda$ to trade-off the feature reconstruction and view-specific contrastive regularization loss. To study the parameter sensitivity, we give the clustering performance measured by ACC, NMI, ARI, and Fscore on eight multi-view datasets varying with different $\lambda$ in Fig. \ref{fig:ps}. From the results, we observe that the clustering performance of the BMvC method slightly fluctuates with the $\lambda$. Additionally, when the parameter $\lambda$ is in range $[10^0, 10^1, 10^2, 10^3]$, our BMvC method can obtain considerable clustering results in most datasets. Therefore, we suggest to set the parameter $\lambda$ in range $[10^0, 10^1, 10^2, 10^3]$, when it is applied for some applications.
		
		\section{Conclusions}
		In this paper, we proposed a novel balanced multi-view clustering method to achieve more balanced multi-view learning and further improve the clustering performance. We first analyzed the imbalanced multi-view clustering problem existing in the joint-training paradigm from the gradient view. Then, we introduced a view-specific contrastive regularization to make a better trade-off between the exploitation of view-specific patterns and the exploration of view-invariance patterns to fully learn the multi-view information for the clustering task. Additionally, the theoretical analysis was provided to verify the effectiveness of the VCR from the gradient perspective. Finally, Extensive experiments on various benchmark multi-view clustering datasets were conducted to verify the efficacy of our method. 
		
		\ifCLASSOPTIONcompsoc
		\section*{Acknowledgments}
		\else
		\section*{Acknowledgment}
		\fi

		The authors wish to gratefully acknowledge the anonymous reviewers for the constructive comments of this paper.

		\ifCLASSOPTIONcaptionsoff
		\newpage
		\fi
		
		\bibliographystyle{IEEEtran}
		\bibliography{IEEEabrv,References}

	\end{document}